%% file: paper.tex
\documentclass[10pt]{fairmeta}
\usepackage{lmodern}

\input{math_commands.tex}

\usepackage{amssymb,amsthm}
\usepackage{graphicx}
\usepackage{algorithm}
\usepackage{algorithmic}
\usepackage{booktabs}

\usepackage{titletoc}
\usepackage{hyperref}
\usepackage{url}
\usepackage{capt-of}
\usepackage{wrapfig,booktabs}

\definecolor{TUCOSectionLink}{HTML}{1F4E79}
\definecolor{TUCOEquationLink}{HTML}{B22222}
\definecolor{TUCOTheoremLink}{HTML}{B22222}
\definecolor{TUCOCitationLink}{HTML}{006600}
\newcommand{\appsectionref}[1]{\hyperref[#1]{\textcolor{TUCOSectionLink}{\ref*{#1}}}}
\newcommand{\appequationref}[1]{\hyperref[#1]{\textcolor{TUCOEquationLink}{\ref*{#1}}}}
\newcommand{\apptheoremref}[1]{\hyperref[#1]{\textcolor{TUCOTheoremLink}{\ref*{#1}}}}

\hypersetup{
  colorlinks=true,
  linkcolor=TUCOSectionLink,
  citecolor=black,
  urlcolor=black
}

\newtheorem{theorem}{Theorem}

\tcbset{methodbox/.style={
  breakable,
  colback=black!7,
  colframe=black!7,
  boxrule=0pt,
  boxsep=2pt,
  left=2pt,
  right=2pt,
  top=2pt,
  bottom=2pt,
  sharp corners,
  before skip=4pt,
  after skip=4pt
}}
\tcolorboxenvironment{theorem}{methodbox}
\tcolorboxenvironment{restatedtheorem}{methodbox}

\title{TUCO: Curating Simulation Demonstrations for Sim-to-Real Robot Policy Co-Training}
\hypersetup{
  pdftitle={TUCO: Curating Simulation Demonstrations for Sim-to-Real Robot Policy Co-Training},
  pdfauthor={Ning Zhu, Mengfei Zhao, Yikai Tang, Zhangyujie Sun, Peihao Li, Dongyue Ni, Jindou Jia, Jianfei Yang},
  pdfsubject={Robot learning research preprint},
  pdfkeywords={robot learning, sim-to-real, data curation, co-training, TUCO}
}
\input{authors.tex}

\abstract{
Simulation demonstrations can supplement scarce real-world data for robot policy co-training.
However, the value of using data curation to actively select these demonstrations for sim-to-real co-training remains underexplored. Existing curation methods also lack a unified criterion for measuring trajectory-level utility and set-level coverage from closed-loop target behavior. To address these gaps, we present the first systematic study of data curation for sim-to-real robot policy co-training and propose \textbf{T}rajectory-level
\textbf{U}tility and set-level \textbf{C}overage \textbf{O}ptimization (\textbf{TUCO}). \textsc{TUCO} uses influence functions to trace how each source demonstration affects target-domain scoring rollouts. 
Our key insight is that these effects can be decomposed into an overall contribution to target return and variation across rollouts, providing a common closed-loop basis for measuring trajectory utility and set coverage. We further propose a performance-aligned subset optimizer that combines these measures in a unified curation objective to reduce redundancy and select complementary demonstrations. Extensive experiments on RoboMimic and OmniReset establish the value of active simulation data curation for sim-to-real policy co-training and show that \textsc{TUCO} achieves state-of-the-art performance across single-simulator, sim-to-sim, and sim-to-real settings.
}
\begin{document}
\maketitle

\section{Introduction}
\label{sec:intro}

Imitation learning is essential for robot manipulation, but the amount of available training data is limited by the cost of real-world teleoperation~\citep{mandlekar2022whatmatters,belkhale2023dataquality,lin2025scalinglaws,saxena2025largedatasets,jia2026mars,jia2026action,zhao2026apex}. 
Automated data generation in simulation offers a scalable alternative~\citep{mandlekar2023mimicgen,nasiriany2024robocasa,hou2026worldmodelrobotlearning}, motivating sim-to-real co-training that combines abundant simulation demonstrations from a source domain with limited real demonstrations from the target
deployment domain~\citep{maddukuri2025simreal,cheng2025domainadaptation}. However, domain shifts make source demonstrations differ in their utility to the target policy, while large source pools may include redundant or simulation-specific behaviors. Directly using the full pool can therefore increase training costs and introduce negative transfer~\citep{dass2026datamil,wei2026ambient}. This raises a central but unaddressed question for sim-to-real co-training: \textit{which simulation demonstrations should be selected to improve target-domain policy performance?}

Robot data curation aims to improve policy performance by constructing training subsets that exclude unhelpful or redundant demonstrations, with existing methods broadly falling into three groups. (1) \textbf{Quality-based approaches} filter noisy demonstrations using data statistics and execution cues~\citep{belkhale2023dataquality,hejna2025deminf,zhang2026scizor,sojib2026efficient}. They offer efficient offline selection but provide no direct link to closed-loop performance~\citep{bedi2026metrics}. (2) \textbf{Utility-based approaches} estimate downstream value from target similarity, validation loss, or policy behavior~\citep{chen2025demoscore,du2023behavior,memmel2025strap,dass2026datamil,lee2026quality}.
\textbf{Influence Function}-based methods are highly promising since they can estimate how reweighting a training example changes the learned model and a downstream metric~\citep{koh2017influence}. CUPID~\citep{agia2025cupid} applies influence functions to closed-loop evaluation by aggregating each demonstration’s effects across policy rollouts with their returns.  
Although the resulting scalar utility connects data selection to target performance, independent ranking does not account for relationships within the selected subset and can select redundant demonstrations while overlooking complementary ones. (3) \textbf{Diversity-based approaches} explicitly model these relationships to reduce redundancy and improve coverage~\citep{zhang2026scizor,sirigiri2026faktual,jia2024diverse}, but their criteria typically measure
differences in representation space without grounding them in closed-loop target behavior. Under
domain shift, such differences may reflect simulator-specific variation that offers little support for target-policy execution. These limitations yield another important question: \textit{How can trajectory-level utility and set-level coverage be jointly optimized in closed-loop performance?}

Our key observation is that the influence of a demonstration across closed-loop target-policy rollouts can be explicitly decomposed into two signals: an overall contribution to target return and the remaining variation across rollouts, providing a common closed-loop basis for measuring trajectory-level utility and set-level coverage. The overall contribution indicates how much a demonstration supports target performance, while relationships among the remaining components reveal redundancy and complementarity across demonstrations. However, since variation across rollouts alone does not imply a positive contribution to target performance, subset selection should retain demonstrations whose effects align with target success and favor candidates that add new closed-loop support beyond the current subset.

Beyond these methodological limitations, existing data curation research has a limited evaluation scope. Most studies focus on single-domain settings or transfer between simulators, or among real-world datasets~\citep{du2023behavior,lin2025flowretrieval,memmel2025strap,hejna2025remix,dass2026datamil,yang2026analogies}, leaving sim-to-real policy co-training underexplored. Meanwhile, sim-to-real co-training studies primarily examine mixing ratios and training strategies~\citep{maddukuri2025simreal,cheng2025domainadaptation,wei2026ambient}, with limited attention to which simulation demonstrations best support real-world policy learning. We aim to extend data curation from single-simulator and sim-to-sim settings to sim-to-real co-training through a unified protocol. Within this
broader evaluation, we explore whether actively selecting simulation demonstrations improves real-world policy performance when the numbers of source and target demonstrations are held fixed.

To our knowledge, we conduct the first systematic study of simulation data curation for sim-to-real robot policy co-training and introduce \textbf{T}rajectory-level \textbf{U}tility and set-level \textbf{C}overage
\textbf{O}ptimization (\textbf{TUCO}), a novel set-level curation framework grounded in target-policy closed-loop behavior. First, we present a closed-loop influence decomposition that separates each
demonstration’s overall contribution to target return from its variation across target-domain scoring rollouts, providing a common basis for measuring trajectory-level utility and set-level coverage. Second, we introduce a performance-aligned subset optimizer that combines both measures in a unified curation objective. It favors demonstrations that support target success and complement the current subset, reducing redundancy during selection. We establish a unified evaluation protocol on RoboMimic~\citep{mandlekar2022whatmatters} and OmniReset~\citep{yin2026omnireset} spanning single-simulator, sim-to-sim, and sim-to-real settings and show that \textsc{TUCO} achieves state-of-the-art curation performance. More importantly, 
we show that active simulation data curation improves sim-to-real policy co-training: TUCO improves over full-pool co-training by 65.0, 35.0, and 80.0 percentage points ($\mathrm{pp}$) across the three real-world tasks and outperforms random selection with the same numbers of source and target demonstrations.

\section{Related Work}
\label{sec:related-work}

\textbf{Data Curation for Robot Learning.}
Existing approaches assess demonstrations through quality, utility, or diversity. Quality-based methods use
statistical dependence~\citep{hejna2025deminf} and execution cues~\citep{zhang2026scizor,sojib2026efficient} to estimate demonstration quality. Utility-based methods condition selection on target-data similarity~\citep{du2023behavior,memmel2025strap} or estimated training contributions~\citep{chen2025demoscore,dass2026datamil}, with influence functions relating these contributions to
imitation loss~\citep{lee2026quality} and closed-loop return~\citep{agia2025cupid,xu2026athena}. Diversity-based
methods construct subsets through deduplication and trajectory diversity~\citep{zhang2026scizor,sirigiri2026faktual} or distribution matching~\citep{feng2025tarot}.

\textbf{Cross-Domain Robot Co-Training.}
Simulation data generation~\citep{mandlekar2023mimicgen,nasiriany2024robocasa, zhao2026axisgrowablecommunitydrivendata} supports co-training with limited
real demonstrations. Composition-oriented studies examine simulation scale and mixing ratios~\citep{wei2025empirical,maddukuri2025simreal}, while alignment methods improve source-target compatibility through observation-action alignment~\citep{cheng2025domainadaptation} and visual
invariance~\citep{yang2025invariance}. Complementary approaches regulate source contributions through learned
mixture weights~\citep{hejna2025remix,sima2026promix} and training schedules~\citep{wei2026ambient}.

\section{Preliminaries: Data Attribution via Influence Functions}
\label{sec:preliminary}
Effective data curation depends on understanding how each training sample affects the output. \textbf{Data Attribution} quantifies this dependence by estimating how changes to training data affect a model's performance. Evaluating these effects through repeated retraining is expensive, so \textbf{Influence Functions} estimate them locally by perturbing the weight of a training example~\citep{koh2017influence}. Consider a training set $\mathcal{D}:=\{z_i\}_{i=1}^{n}$ with per-example loss $\ell(z;\theta)$ and parameters $\theta\in\mathbb R^{d_\theta}$. The learned parameters $\theta_0:=\theta(\mathcal{D})$ minimize $\mathcal{L}(\theta;\mathcal{D}):=n^{-1}\sum_{i=1}^{n}\ell(z_i;\theta)$. Increasing the weight of an example $z$ by $\epsilon$ gives the perturbed objective $\mathcal{L}(\theta;\mathcal{D})+\epsilon\ell(z;\theta)$, whose local minimizer near $\theta_0$ is denoted by $\theta_{\epsilon,z}$. Let $f(\hat{z};\theta)$ be a differentiable metric evaluated on an example $\hat{z}$ and $\mathbf{H}_0:=\nabla_{\theta}^{2}\mathcal{L}(\theta_0;\mathcal{D})\in\mathbb R^{d_\theta\times d_\theta}$ the training-loss Hessian. With twice continuously differentiable training losses and a nonsingular $\mathbf{H}_0$, the influence of $z$ on this metric is
$\Psi_{\mathrm{inf}}(\hat{z},z)
:=
\left.\partial_{\epsilon}f(\hat{z};\theta_{\epsilon,z})\right|_{\epsilon=0}
=
-\nabla_{\theta}f(\hat{z};\theta_0)^{\top}
\mathbf{H}_0^{-1}
\nabla_{\theta}\ell(z;\theta_0)$.
This measures how upweighting a training example changes the evaluation metric through its effect on the learned parameters.

\section{Methodology}
\label{sec:method}

\begin{figure}[!t]
    \centering
    \includegraphics[width=\textwidth]{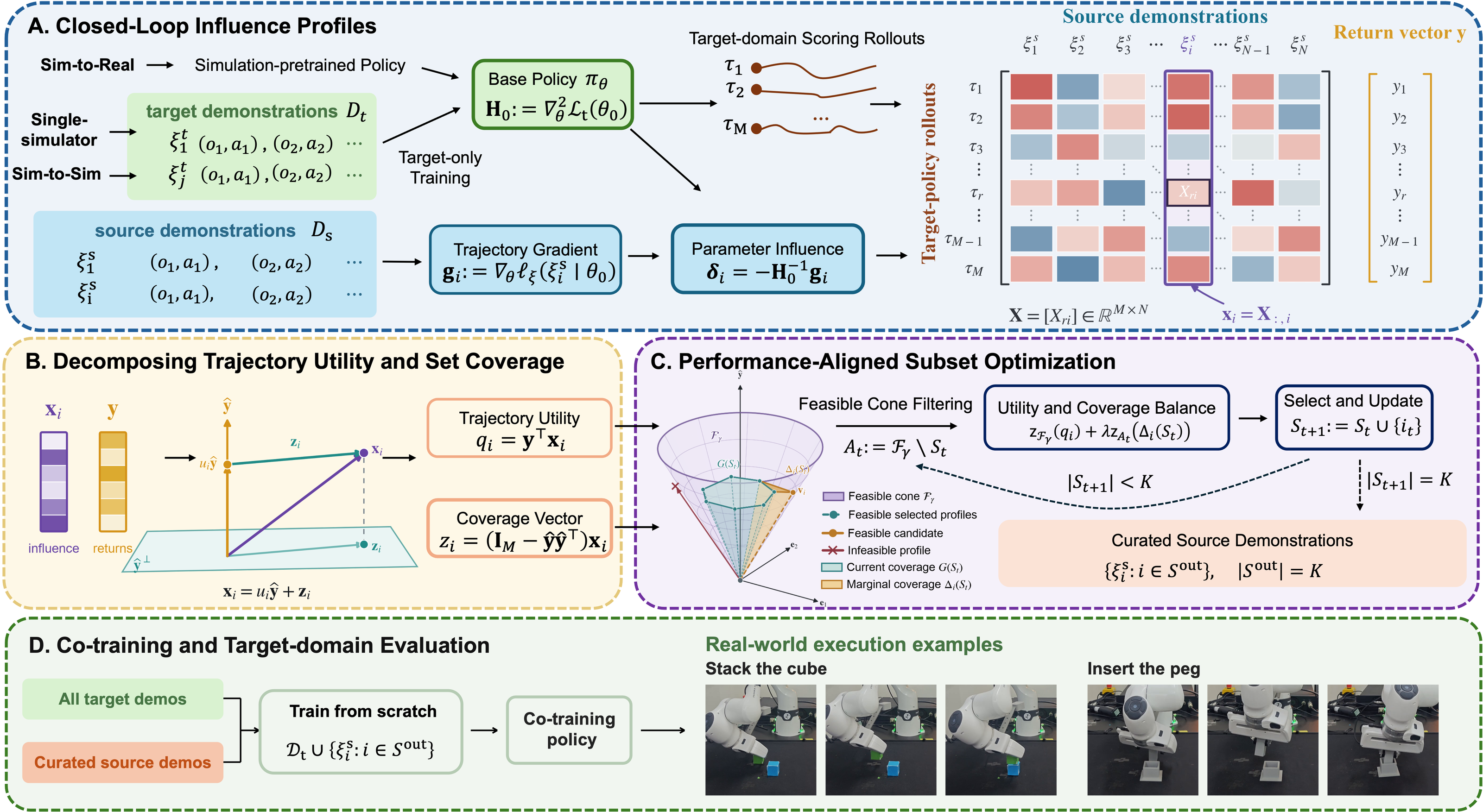}
    \caption{\textbf{Overview of TUCO.}
  (\textbf{A}) Influence functions estimate source demonstration effects on target-domain scoring rollouts at a fixed base policy, yielding the influence matrix $\mathbf X$ and return vector $\mathbf y$.
  (\textbf{B}) Each influence profile is decomposed for trajectory utility and set coverage.
  (\textbf{C}) TUCO filters candidates by performance alignment and iteratively selects trajectories using utility and marginal coverage.
  (\textbf{D}) The curated source subset is combined with target demonstrations for co-training and target-domain evaluation.}
    \label{fig:main}
\end{figure}

 In this section, we present \textsc{TUCO} for source demonstration selection in target-domain policy co-training. First, we formulate the source-to-target co-training problem. Second, we decompose closed-loop influence profiles to measure trajectory utility and set coverage. Third, we combine both measures in a performance-aligned subset optimizer. Figure~\ref{fig:main} provides an overview.

\subsection{Problem Formulation}
\label{sec:problem-formulation}

\paragraph{Imitation Learning.}
We consider a source environment $\mathcal E_{\mathrm s}$ and a target environment $\mathcal E_{\mathrm t}$ with compatible observation and action interfaces but potentially different observations or dynamics. A demonstration $\xi=((o_h,a_h))_{h=0}^{T_\xi-1}$ is an observation-action trajectory, where $o_h\in\mathcal O\subseteq\mathbb R^{d_o}$ and $a_h\in\mathcal A\subseteq\mathbb R^{d_a}$. Given a demonstration set $\mathcal D$, behavior cloning~\citep{argall2009survey,osa2018algorithmic} trains a policy $\pi_\theta$ with $\theta\in\mathbb R^{d_\theta}$ by minimizing
$
\theta(\mathcal D)
\in
\arg\min_{\theta'}
\left\{
\mathcal L_{\mathrm{IL}}(\theta',\mathcal D)
:=
\mathbb E_{(o,a)\sim\nu_{\mathcal D}}
\left[\ell(o,a\mid\theta')\right]
\right\}.
$ $\nu_{\mathcal D}$ is the observation-action sampling distribution. We evaluate the policy in $\mathcal E_{\mathrm t}$ through its expected return $J_{\mathrm t}(\theta):=\mathbb E_{\tau\sim p_{\mathrm t}(\tau\mid\pi_\theta)}[R_{\mathrm t}(\tau)]$, where $p_{\mathrm t}$ denotes the target rollout distribution and $R_{\mathrm t}$ the task return. For binary tasks, we assign return $+1$ to success and $-1$ to failure.

\paragraph{Source-to-Target Data Curation.}
Let $\mathcal D_{\mathrm t}:=\{\xi_j^{\mathrm t}\}_{j=1}^{n_{\mathrm t}}$ denote the fixed target demonstrations and $\mathcal D_{\mathrm s}:=\{\xi_i^{\mathrm s}\}_{i=1}^{N}$ the source candidate pool. Given a budget $K\leq N$, we select an index set $S\subseteq[N]:=\{1,\ldots,N\}$ of size $K$ and combine the corresponding trajectories $\mathcal D_{\mathrm s}(S):=\{\xi_i^{\mathrm s}:i\in S\}$ with all target demonstrations for policy training. In sim-to-real comparisons, we hold the source-to-target sampling ratio fixed across curation methods. The curation objective is
$
S^\star
\in
\arg\max_{\substack{S\subseteq[N],|S|=K}}
J_{\mathrm t}
\left(
\theta\!\left(
\mathcal D_{\mathrm t}\cup\mathcal D_{\mathrm s}(S)
\right)
\right).
\label{eq:curation-objective}
$

This formulation covers three source-to-target settings: (1) \textbf{Single-Simulator}, where all demonstrations come from the same simulated environment. We evaluate both selection,
  which augments a fixed target set with selected candidates, and filtering, which retains a subset of an existing training set; (2) \textbf{Sim-to-Sim}, where source and target
  demonstrations come from different simulators; and (3) \textbf{Sim-to-Real}, where source demonstrations come from simulation and target demonstrations from the physical environment. 
  

\subsection{Closed-Loop Influence Profiles}
\label{sec:closed-loop-attribution}
We first obtain a base policy $\pi_{\theta_0}$ and use its closed-loop executions in the target environment $\mathcal E_{\mathrm t}$ as feedback for curation. In single-simulator and sim-to-sim settings, we train this policy on target demonstrations and collect its rollouts, using the original unfiltered training set for single-simulator filtering. In sim-to-real, we use a simulation-pretrained policy and reuse its successful rollouts recorded during real-robot deployment. We keep the base policy fixed and score all source candidates against the same set of target-domain scoring rollouts $\mathcal R_{\mathrm t}:=\{\tau_r\}_{r=1}^{M}$. We use influence functions to estimate how each source demonstration affects these target executions. For a source trajectory of length $T_i$, let $\ell_\xi(\xi_i^{\mathrm s}\mid\theta):=\sum_{h=0}^{T_i-1}\ell(o_{i,h}^{\mathrm s},a_{i,h}^{\mathrm s}\mid\theta)$ denote its training loss. We first derive the profile for a target-trained base policy $\theta_0:=\theta(\mathcal D_{\mathrm t})$. Increasing the source trajectory's weight in the target training objective $\mathcal L_{\mathrm t}(\theta):=\mathcal L_{\mathrm{IL}}(\theta,\mathcal D_{\mathrm t})$ gives the local minimizer $\theta_{\epsilon,i}$ of $\mathcal L_{\mathrm t}(\theta)+\epsilon\ell_\xi(\xi_i^{\mathrm s}\mid\theta)$ near $\theta_0$. For a target rollout $\tau_r$, the influence expression gives
$
X_{ri}
:=
\left.\partial_\epsilon
\log p_{\mathrm t}(\tau_r\mid\pi_{\theta_{\epsilon,i}})
\right|_{\epsilon=0}
=
-\mathbf h_r^\top\mathbf H_0^{-1}\mathbf g_i,
\label{eq:rollout-wise-influence}
$
where $\mathbf g_i:=\nabla_\theta\ell_\xi(\xi_i^{\mathrm s}\mid\theta_0)\in\mathbb R^{d_\theta}$ is the source trajectory's training gradient, $\mathbf H_0:=\nabla_\theta^2\mathcal L_{\mathrm t}(\theta_0)\in\mathbb R^{d_\theta\times d_\theta}$ is the target training-loss Hessian and $\mathbf h_r:=\left.\nabla_\theta\log p_{\mathrm t}(\tau_r\mid\pi_\theta)\right|_{\theta=\theta_0}\in\mathbb R^{d_\theta}$ is the rollout log-likelihood gradient. A positive $X_{ri}$ means that locally upweighting demonstration $i$ makes rollout $r$ more likely. Collecting these effects over $\mathcal R_{\mathrm t}$ gives the matrix $\mathbf X:=[X_{ri}]\in\mathbb R^{M\times N}$. Its column $\mathbf x_i:=\mathbf X_{:,i}\in\mathbb R^M$ forms the rollout-conditioned influence profile of demonstration $i$. The profile records which executions become more or less likely, and their returns determine resulting performance contributions. With $y_r:=R_{\mathrm t}(\tau_r)$ and $\mathbf y:=[y_1,\ldots,y_M]^\top\in\mathbb R^M$, we obtain
$
q_i:=\mathbf y^\top\mathbf x_i
=\sum_{r=1}^{M}y_rX_{ri}.
\label{eq:raw-utility}
$
For independent rollouts sampled from $p_{\mathrm t}(\cdot\mid\pi_{\theta_0})$, $q_i/M$ estimates $\left.\partial_\epsilon J_{\mathrm t}(\theta_{\epsilon,i})\right|_{\epsilon=0}$ when profile entries are exact and the conditions in Appendix~\ref{app:influence-derivation} hold. In practice, we construct $\mathbf X$ using the projected loss-gradient estimator in Appendix~\ref{app:influence-estimation}. For sim-to-real curation, we evaluate these gradients at the frozen simulation-pretrained policy and set $y_r=1$ for each recorded successful execution, so $q_i$ measures estimated support for the observed successful target behaviors. Our key insight is that each influence profile can be decomposed into a component along the target-return direction and an orthogonal component. The former preserves trajectory-level utility, while the latter captures differences in rollout influence among demonstrations. This decomposition provides a common closed-loop basis for measuring individual utility and complementary coverage within a selected set.

\subsection{Decomposing Trajectory Utility and Set Coverage}
\label{sec:influence-decomposition}

The utility $q_i$ supports trajectory ranking through CUPID's linear subset objective $F_{\mathrm{CUPID}}(S):=\sum_{i\in S}q_i$, which is maximized by Top-$K$ selection. However, aggregation into $q_i$ leaves differences orthogonal to the return vector unrepresented, so distinct rollout influence profiles can receive the same utility. Each trajectory also contributes a fixed $q_i$ regardless of the selected set, preventing these differences from influencing whether trajectories are selected together. We therefore introduce a closed-loop influence decomposition that preserves trajectory utility while exposing the remaining differences across target-domain scoring rollouts. Let $\widehat{\mathbf y}:=\mathbf y/\|\mathbf y\|_2\in\mathbb R^M$ and let $\mathbf I_M\in\mathbb R^{M\times M}$ denote the identity matrix. We decompose
\begin{equation}
\mathbf x_i
=
\underbrace{u_i\widehat{\mathbf y}}_{\text{utility component}}
+
\underbrace{\mathbf z_i}_{\text{coverage representation}},
\qquad
u_i:=\widehat{\mathbf y}^{\top}\mathbf x_i,
\qquad
\mathbf z_i
:=
\left(
\mathbf I_M-\widehat{\mathbf y}\widehat{\mathbf y}^{\top}
\right)\mathbf x_i.
\label{eq:influence-decomposition}
\end{equation}
The coordinate $u_i$ preserves trajectory-level utility through $q_i=\|\mathbf y\|_2u_i$. The residual $\mathbf z_i\in\mathbb R^M$ captures how trajectories with similar utility affect scoring rollouts differently, providing a basis for measuring their redundancy and complementarity. We use these residual directions to measure the additional coverage a candidate provides relative to the selected set.

\subsection{Performance-Aligned Subset Optimization}
\label{sec:subset-optimization}
The residuals reveal differences across target-domain scoring rollouts, but covering these differences is useful only when the selected demonstrations also support target performance. We therefore develop a performance-aligned subset optimizer to combine this coverage with trajectory utility. We first restrict coverage optimization to candidates with positive estimated trajectory-utility and sufficient alignment with $\mathbf y$, defining the feasible set
\begin{equation}
\mathcal F_\gamma:=
\left\{i\in[N]:
q_i>0,\ \|\mathbf x_i\|_2>0,\
\frac{\mathbf x_i^\top\mathbf y}{\|\mathbf x_i\|_2\|\mathbf y\|_2}\geq\gamma
\right\},
\label{eq:feasible-cone}
\end{equation}
where $\gamma\in[0,1)$ controls the width of the cone around the target-return direction. Through this constraint, we require the return-aligned component to account for at least a fraction $\gamma$ of the influence-profile norm. Within this feasible pool, we aim to improve set-level coverage by selecting demonstrations that contribute influence directions poorly represented in the current subset. The intuition is to view coverage as the volume spanned by the selected residual vectors, so a demonstration contributes through the additional volume it introduces. To focus on directional coverage, we normalize residuals to $\mathbf v_i:=\mathbf z_i/\|\mathbf z_i\|_2\in\mathbb R^M$. We also give higher-utility trajectories greater weight through $w_i:=q_i/\kappa_{\mathcal F}$, where $\kappa_{\mathcal F}:=|\mathcal F_\gamma|^{-1}\sum_{j\in\mathcal F_\gamma}q_j$. We formalize this intuition with a quality-weighted, regularized log-determinant criterion and its marginal gain:
\begin{equation}
G(S):=\log\det\!\left(\mathbf I_M+\rho\sum_{i\in S}w_i\mathbf v_i\mathbf v_i^\top\right),
\qquad
\Delta_i(S):=G(S\cup\{i\})-G(S),
\label{eq:logdet-coverage}
\end{equation}
where $S\subseteq\mathcal F_\gamma$ and $\rho>0$ controls the coverage scale. For comparable utility weights, this criterion gives smaller gains to directions already well represented in $S$. The identity term keeps the measure well defined for subsets of any size. We provide the precise volume relationship and efficient computation in Appendix~\ref{app:subset-computation}. Because a candidate's additional coverage depends on the demonstrations already selected, we update its marginal gain as the subset grows. We combine this gain with trajectory utility to favor demonstrations that are both useful and complementary. To put the two terms on comparable scales, we use population z-scores $\mathsf z_A$ over the fixed feasible pool for utility and the remaining candidates for coverage. We assign zero when the corresponding variance vanishes. Starting from $S_0=\varnothing$, we select from $A_t:=\mathcal F_\gamma\setminus S_t$ according to
\begin{equation}
i_t\in\argmax_{i\in A_t}
\left[
\mathsf z_{\mathcal F_\gamma}(q_i)
+\lambda\mathsf z_{A_t}\!\left(\Delta_i(S_t)\right)
\right],
\qquad
S_{t+1}:=S_t\cup\{i_t\},
\label{eq:greedy-rule}
\end{equation}
where $\lambda\geq0$ balances utility and coverage. We keep the influence profiles fixed throughout selection and recompute marginal gains after each addition until the budget or feasible pool is exhausted. When $|\mathcal F_\gamma|<K$, we retain all feasible candidates and fill the remaining slots with the highest-utility unselected trajectories. For co-training, we retrain the policy on all fixed target demonstrations with the selected source demonstrations, $\mathcal D_{\mathrm t}\cup\mathcal D_{\mathrm s}(S^{\mathrm{out}})$, see Algorithm~\ref{alg:subset-selection} for summary.

\begin{algorithm}[htbp]
\caption{\textsc{TUCO}: Trajectory-Level Utility and Set-Level Coverage Optimization}
\label{alg:subset-selection}
\small
\begin{algorithmic}[1]
\REQUIRE Influence matrix $\mathbf X\in\mathbb R^{M\times N}$, returns $\mathbf y\in\mathbb R^M$, budget $K$, parameters $\gamma$, $\lambda$, $\rho$
\ENSURE Selected source indices $S^{\mathrm{out}}\subseteq[N]$
\STATE $\widehat{\mathbf y}\leftarrow\mathbf y/\|\mathbf y\|_2$
\STATE $q_i\leftarrow\mathbf y^\top\mathbf x_i$, $\mathbf z_i\leftarrow\mathbf x_i-(\widehat{\mathbf y}^\top\mathbf x_i)\widehat{\mathbf y}$ for $i\in[N]$
\STATE $\mathcal F_\gamma\leftarrow\{i\in[N]:q_i>0,\ \|\mathbf x_i\|_2>0,\ \mathbf x_i^\top\widehat{\mathbf y}/\|\mathbf x_i\|_2\geq\gamma\}$
\STATE Normalize residuals to $\mathbf v_i$ and utilities to $w_i$ over $\mathcal F_\gamma$ when nonempty
\STATE $S\leftarrow\varnothing$, $\mathbf M\leftarrow\mathbf I_M\in\mathbb R^{M\times M}$
\WHILE{$|S|<\min(K,|\mathcal F_\gamma|)$}
    \STATE $A\leftarrow\mathcal F_\gamma\setminus S$
    \STATE $\Delta_i\leftarrow\log(1+\rho w_i\mathbf v_i^\top\mathbf M^{-1}\mathbf v_i)$ for $i\in A$
    \STATE $i^\star\leftarrow\argmax_{i\in A}[\mathsf z_{\mathcal F_\gamma}(q_i)+\lambda\mathsf z_A(\Delta_i)]$
    \STATE $\mathbf M\leftarrow\mathbf M+\rho w_{i^\star}\mathbf v_{i^\star}\mathbf v_{i^\star}^\top$, $S\leftarrow S\cup\{i^\star\}$
\ENDWHILE
\IF{$|S|<K$}
    \STATE Add the $K-|S|$ highest-utility indices from $[N]\setminus S$ to $S$
\ENDIF
\RETURN $S^{\mathrm{out}}\leftarrow S$
\end{algorithmic}
\end{algorithm}

\section{Experiments}
\label{sec:experiments}
We conduct a series of experiments to compare TUCO with state-of-the-art robot data curation baselines across Single-Simulator, Sim-to-Sim, and Sim-to-Real co-training. Following CUPID~\citep{agia2025cupid}, the Single-Simulator study uses the \texttt{Lift}, \texttt{Square}, and \texttt{Transport} tasks from RoboMimic's multi-human (MH) task suite~\citep{mandlekar2022whatmatters}. The cross-domain studies use the \texttt{Peg}, \texttt{StackCube}, and \texttt{CupCake} tasks from OmniReset~\citep{yin2026omnireset}.  We aim to answer:
(\textbf{Q1}) Can TUCO enable reliable curation as the domain gap changes? (\textbf{Q2}) Can active data curation improve sim-to-real co-training performance?

\paragraph{Experimental Settings.}
We follow the source-to-target curation problem defined in Section~\ref{sec:method}. In the \textbf{single-simulator setting}, source demonstrations and target-policy evaluation share the same RoboSuite environment. Following CUPID~\citep{agia2025cupid}, we separately partition each 300-demonstration RoboMimic dataset for filtering and selection. Filtering removes demonstrations from an initial set of 192, with 12 reserved for validation. Selection augments 48 fixed target demonstrations with a subset of 240 source candidates, using the remaining 12 for validation. In the \textbf{sim-to-sim setting}, IsaacSim is the source domain and MuJoCo is the target domain. \texttt{Peg} uses 3,000 source and 500 target demonstrations, while \texttt{StackCube} and \texttt{CupCake} each use 5,000 source and 1,000 target demonstrations. In both simulated settings, each base policy provides 100 target-domain scoring rollouts. In the \textbf{sim-to-real setting}, we reuse 10 recorded successful real-world rollouts per task from a frozen simulation-pretrained base policy. These serve as both the target demonstrations for co-training and the scoring rollouts for curation. All curation methods select 600 demonstrations from the same pool of 1,200 simulation demonstrations. For TUCO, we fix the cone threshold at $\gamma=0.01$, the coverage scale $\rho=0.1$ and use $\lambda=0.2$ throughout. We refer the detailed setup to Appendix~\ref{app:experimental-details}.

\paragraph{Baselines and Evaluation Protocol.}
For single-simulator and sim-to-sim settings, we compare with representative methods from the three main families of robot data curation. \textbf{Quality-based} baselines include DemInf~\citep{hejna2025deminf} and PSD~\citep{sojib2026efficient}. \textbf{Utility-based} baselines include Demo-SCORE~\citep{chen2025demoscore}, Success Similarity and CUPID~\citep{agia2025cupid}, QoQ~\citep{lee2026quality}, and DataMIL~\citep{dass2026datamil}. \textbf{Diversity-based} baselines include FAKTUAL~\citep{sirigiri2026faktual}, SIEVE~\citep{wu2026sieve}, and TAROT~\citep{feng2025tarot}. For sim-to-real settings, we compare with Success Similarity, CUPID~\citep{agia2025cupid}, DataMIL~\citep{dass2026datamil} and FAKTUAL~\citep{sirigiri2026faktual}. \footnote{Given real-robot costs, we prioritize baselines closely aligned with target-conditioned subset selection. Demo-SCORE requires unavailable failure feedback. Broader comparisons are conducted in simulation.}We also report random selection together with target-only training and co-training on the full source pool.  Each selected subset is used to train a policy from scratch, and performance is measured by closed-loop success rate in the target domain.

\subsection{Main Results: From Single-Simulator Curation to Sim-to-Real Co-Training}

\paragraph{Single-Simulator.}

We first evaluate TUCO under the standard single-simulator protocols, which isolate data curation from domain shift. Figure~\ref{fig:singlesim-main-results} shows that TUCO achieves the highest or tied-highest mean success rate in 18 of 25 task--budget combinations. On \texttt{Square}, TUCO improves over CUPID by 14.2 $\mathrm{pp}$ after filtering 50\% of the demonstrations and by 15.8 $\mathrm{pp}$ when selecting 10\% of the candidates. Notably, TUCO matches full-data training on \texttt{Lift} with only 25\% of the original training set and surpasses it on \texttt{Square} with 50\%. These results establish the value of curation in a controlled setting and motivate its evaluation in the cross-domain settings that follow.

\begin{figure}[!t]
    \centering
    \includegraphics[width=0.95\textwidth]{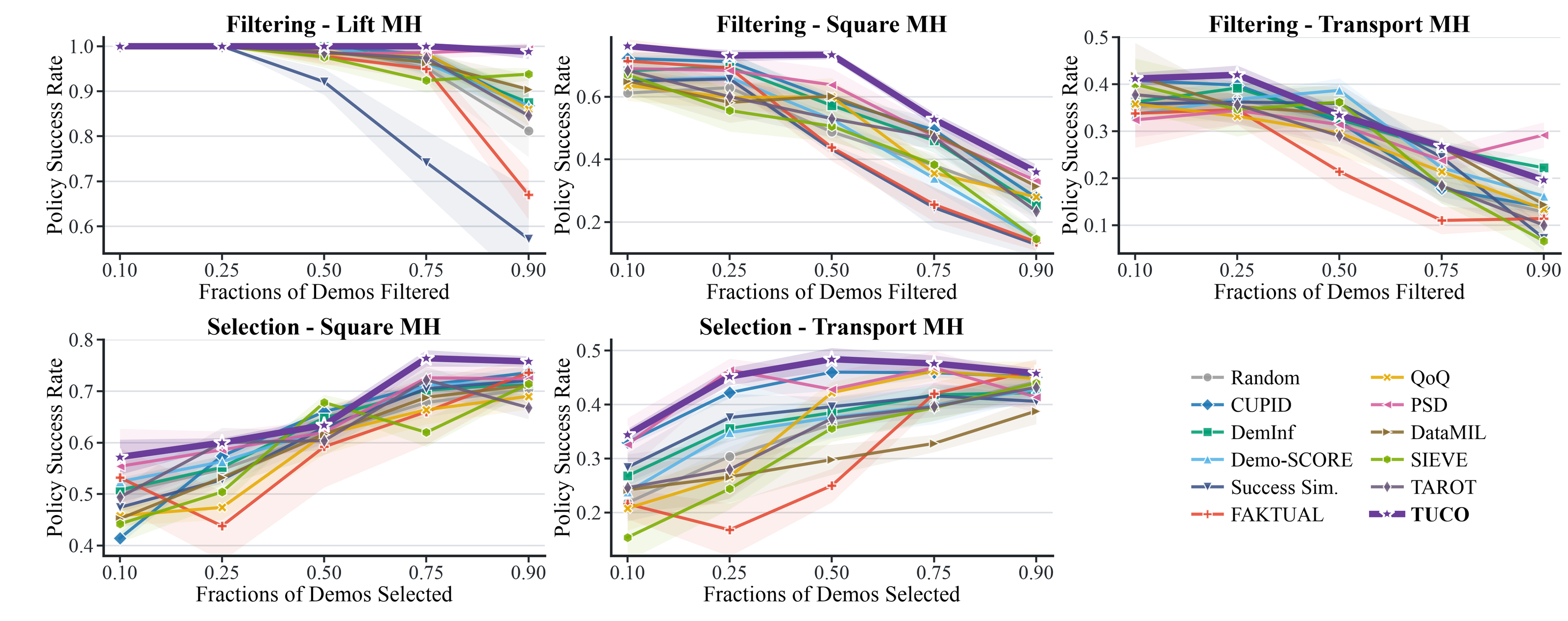}
    \caption{Single-simulator data curation on RoboMimic under demonstration filtering and selection. Results are averaged over 3 random seeds. Success rates are computed over 50 evaluation rollouts from the last 10 checkpoints (500 rollouts total).}
    \label{fig:singlesim-main-results}
\end{figure}

\paragraph{Sim-to-Sim Co-Training.}
\begin{figure}[!t]
    \centering
    \includegraphics[width=0.95\textwidth]{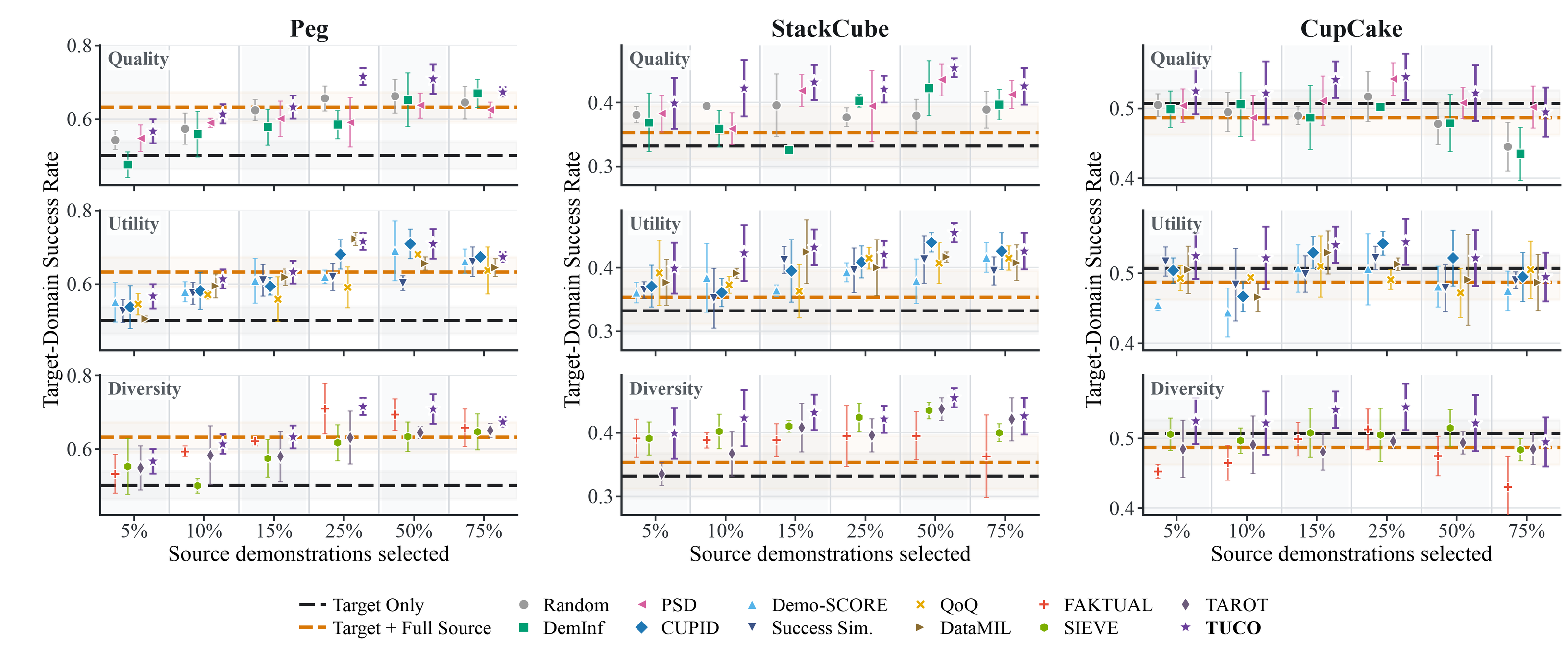}
    \caption{Sim-to-Sim co-training with curated IsaacSim demonstrations and fixed MuJoCo target data across different source-data budgets. Results are averaged over 3 random seeds. Success rates are computed over 50 evaluation rollouts from the last 5 checkpoints (250 rollouts total).}
    \label{fig:sim-to-sim-main-results}
\end{figure}

We first examine cross-domain curation in Sim-to-Sim as a controlled
setting with observation and dynamics mismatch. Figure~\ref{fig:sim-to-sim-main-results} shows that TUCO achieves the highest mean success rate in 15 of the 18 task--budget combinations and outperforms random selection in all 18. Its best curated subsets exceed full-pool co-training by 8.4 $\mathrm{pp}$ on \texttt{Peg}, 10.2 $\mathrm{pp}$ on \texttt{StackCube}, and 5.8 $\mathrm{pp}$ on \texttt{CupCake}. Full-pool co-training improves \texttt{Peg}, offers only a modest gain on \texttt{StackCube}, and causes negative transfer on \texttt{CupCake}. In contrast, each task reaches its best performance with a curated subset. This shows that successful source demonstrations can differ substantially in their target-domain value under simulator shift. The comparison with random selection further shows that the gain comes from identifying suitable trajectories rather than merely reducing the amount of source data. Quality-based methods remain inconsistent because source-domain quality does not directly indicate transfer value. Utility-based methods better reflect the target task but can retain redundant trajectories through independent ranking. Diversity-based methods perform well at some budgets, although source-space variation does not always support target-domain behavior. TUCO is more consistent by grounding both utility and coverage in target-domain scoring rollouts. 

\paragraph{Sim-to-Real Co-Training.}

\begin{table}[t!]
    \centering
    \caption{Success rates (\%) across three real-world tasks.
    Each method is evaluated over 40 trials per task (120 real-robot trials in total).
    Avg. denotes the macro-average across tasks.}
    \label{tab:sim2real-tasks}
    \small
    \setlength{\tabcolsep}{7pt}
    \begin{tabular}{@{}lcccc@{}}
        \toprule
        Method / training data 
        & \texttt{Peg} 
        & \texttt{StackCube} 
        & \texttt{CupCake}
        & \textbf{Avg.} \\
        \midrule

        \multicolumn{5}{@{}l}{\emph{Training-data baselines}} \\
        10 real demos + all sim data
            & 17.5 & 40.0 & 20.0 & 25.8 \\
        10 real demos
            & 15.0 & 35.0 & 0.0 & 16.7 \\

        \addlinespace
        \multicolumn{5}{@{}l}{\emph{Compared methods}} \\
        Random
            & 7.5 & 57.5 & 12.5 & 25.8 \\
        CUPID
            & 75.0 & 65.0 & 80.0 & 73.3 \\
        DataMIL
            & 22.5 & 65.0 & \textbf{100.0} & 62.5 \\
        FAKTUAL
            & 10.0 & \textbf{75.0} & \textbf{100.0} & 61.7 \\
        Success Similarity
            & 5.0 & 50.0 & 60.0 & 38.3 \\
        TUCO (ours)
            & $\mathbf{82.5}^{\dagger\ddagger}$ & $\mathbf{75.0}^{\dagger}$ & $\mathbf{100.0}^{\dagger\ddagger}$ & \textbf{85.8} \\

        \bottomrule
    \end{tabular}
    \par\smallskip
    \begin{minipage}{\linewidth}
        \footnotesize
        For TUCO, $\dagger$ and $\ddagger$ indicate significant improvements over full-pool co-training and random selection, respectively (two-sided Fisher's exact tests with Holm correction across six comparisons, adjusted $p<0.05$). 
    \end{minipage}
\end{table}

Finally, we examine whether active curation improves co-training
on physical manipulation tasks. Table~\ref{tab:sim2real-tasks} shows that TUCO achieves the highest success rate on all three tasks.
Although FAKTUAL matches TUCO on \texttt{StackCube} and \texttt{CupCake}, its success rate is 72.5 $\mathrm{pp}$ lower on \texttt{Peg}, highlighting the consistency of TUCO across tasks. More importantly, TUCO improves over full-pool co-training by 65.0 $\mathrm{pp}$, 35.0 $\mathrm{pp}$, and 80.0 $\mathrm{pp}$ on \texttt{Peg}, \texttt{StackCube}, and \texttt{CupCake}, respectively. With the same numbers of source and target demonstrations, TUCO also outperforms random selection by 75.0 $\mathrm{pp}$, 17.5 $\mathrm{pp}$, and 87.5 $\mathrm{pp}$ supporting source-data composition as a consequential design choice in sim-to-real co-training. These gains are obtained using only 10 recorded successful real-world rollouts per task from a simulation-pretrained base policy. While the simulated experiments collect scoring rollouts from a target-trained policy, the real-world experiments reuse these existing executions as both scoring feedback and target demonstrations. Even with identical success labels, their influence profiles retain differences in how source demonstrations support individual executions, allowing utility and coverage to guide selection. The results therefore show that limited successful real-world experience can provide useful feedback for simulation data curation without additional scoring rollouts. 

\subsection{Ablation Results}

We further ablate TUCO to examine the contributions of trajectory utility, set coverage, and the cone constraint. \textbf{Utility} uses closed-loop utility ranking and corresponds to CUPID. \textbf{Coverage + Cone} removes the additive utility term and selects by quality-weighted marginal coverage. Once the cone is exhausted,it continues this selection outside the cone, retaining signed utility
weights. \textbf{Utility + Coverage} sets $\gamma=0$ while retaining both terms in the selection objective. Table~\ref{tab:component-ablation} compares these variants with TUCO, where $U$, $C$, and $G$ denote utility, coverage, and the cone constraint. TUCO achieves the highest in 15 of the 18 settings. Compared with $U$, it improves mean performance by 2.1 $\mathrm{pp}$ on average. The gain increases to 3.0 $\mathrm{pp}$ for budgets up to $25\%$, suggesting that coverage is particularly useful when selection is restrictive. The coverage-driven variant yields lower mean performance in 16 settings. Setting $\gamma=0$ lowers it in 12 settings. These results indicate that coverage benefits from direct utility guidance and a performance-aligned candidate pool. At the $75\%$ budget, TUCO and $U+C$ exhaust their feasible pools and complete selection by utility ranking, recovering the same subsets as $U$. The coverage-based completion in $C+G$ can produce different subsets at this budget. Figure~\ref{fig:cone-distribution} helps explain why the angular margin has a clear effect despite changing the feasible pool by just $4.1\%$ on \texttt{peg} and $3.6\%$ on \texttt{stackcube}. The excluded candidates are nearly orthogonal to target return, so their residual components can yield large coverage gains while providing little support for target success. Removing these misleading directions keeps coverage focused on performance-aligned behavior.

\begin{figure}[t!]
    \centering
    \begin{minipage}[t]{0.50\textwidth}
        \vspace{0pt}
        \centering
          \captionof{table}{Component ablations in sim-to-sim co-training. We report success rates (\%).}
        \label{tab:component-ablation}
        \vspace{1pt}
        \footnotesize
        \setlength{\tabcolsep}{3.2pt}
        \renewcommand{\arraystretch}{1.10}
        \resizebox{\linewidth}{!}{%
        \begin{tabular}{c|cccc}
            \toprule
            Budget (\%) & $U$ & $C{+}G$ & $U{+}C$ & TUCO \\
            \midrule
            \multicolumn{5}{l}{\texttt{Peg}} \\
            $5$  & $53.8{\pm}5.8$ & $54.8{\pm}4.1$ & $\mathbf{57.0{\pm}3.9}$ & $56.7{\pm}3.3$ \\
            $10$ & $58.2{\pm}5.0$ & $59.6{\pm}2.1$ & $58.6{\pm}5.7$ & $\mathbf{61.4{\pm}2.6}$ \\
            $15$ & $59.4{\pm}2.4$ & $61.9{\pm}4.8$ & $62.6{\pm}4.1$ & $\mathbf{63.3{\pm}3.1}$ \\
            $25$ & $68.1{\pm}4.0$ & $66.9{\pm}0.8$ & $62.3{\pm}5.1$ & $\mathbf{71.6{\pm}2.3}$ \\
            $50$ & $\mathbf{70.9{\pm}4.0}$ & $68.7{\pm}1.2$ & $62.0{\pm}7.0$ & $\mathbf{70.9{\pm}4.0}$ \\
            $75$ & $\mathbf{67.4{\pm}1.3}$ & $\mathbf{67.4{\pm}2.1}$ & $\mathbf{67.4{\pm}1.3}$ & $\mathbf{67.4{\pm}1.3}$ \\
            \midrule
            \multicolumn{5}{l}{\texttt{StackCube}} \\
            $5$  & $37.1{\pm}3.3$ & $36.1{\pm}5.2$ & $39.3{\pm}5.4$ & $\mathbf{39.9{\pm}4.0}$ \\
            $10$ & $36.1{\pm}2.2$ & $37.1{\pm}4.7$ & $36.0{\pm}0.9$ & $\mathbf{42.3{\pm}4.4}$ \\
            $15$ & $39.5{\pm}4.9$ & $39.4{\pm}1.1$ & $40.0{\pm}0.3$ & $\mathbf{43.2{\pm}2.8}$ \\
            $25$ & $40.9{\pm}2.5$ & $40.9{\pm}2.1$ & $\mathbf{42.3{\pm}1.3}$ & $42.1{\pm}2.1$ \\
            $50$ & $44.0{\pm}1.5$ & $41.4{\pm}3.5$ & $43.6{\pm}2.9$ & $\mathbf{45.5{\pm}1.5}$ \\
            $75$ & $\mathbf{42.6{\pm}2.9}$ & $40.9{\pm}3.3$ & $\mathbf{42.6{\pm}2.9}$ & $\mathbf{42.6{\pm}2.9}$ \\
            \midrule
            \multicolumn{5}{l}{\texttt{CupCake}} \\
            $5$  & $50.4{\pm}1.8$ & $51.6{\pm}3.0$ & $51.4{\pm}1.8$ & $\mathbf{52.5{\pm}3.3}$ \\
            $10$ & $46.7{\pm}2.1$ & $48.5{\pm}3.6$ & $50.6{\pm}1.3$ & $\mathbf{52.2{\pm}4.5}$ \\
            $15$ & $53.0{\pm}2.3$ & $47.7{\pm}2.6$ & $53.1{\pm}2.9$ & $\mathbf{54.1{\pm}2.6}$ \\
            $25$ & $54.3{\pm}1.7$ & $49.3{\pm}2.1$ & $50.7{\pm}2.4$ & $\mathbf{54.5{\pm}3.3}$ \\
            $50$ & $52.2{\pm}4.0$ & $\mathbf{52.4{\pm}2.5}$ & $52.2{\pm}4.0$ & $52.2{\pm}4.0$ \\
            $75$ & $\mathbf{49.5{\pm}3.5}$ & $45.5{\pm}2.4$ & $\mathbf{49.5{\pm}3.5}$ & $\mathbf{49.5{\pm}3.5}$ \\
            \bottomrule
        \end{tabular}%
        }
    \end{minipage}\hspace{0.012\textwidth}%
    \begin{minipage}[t]{0.475\textwidth}
        \vspace{2pt}
        \raggedright
        \includegraphics[width=0.92\linewidth]{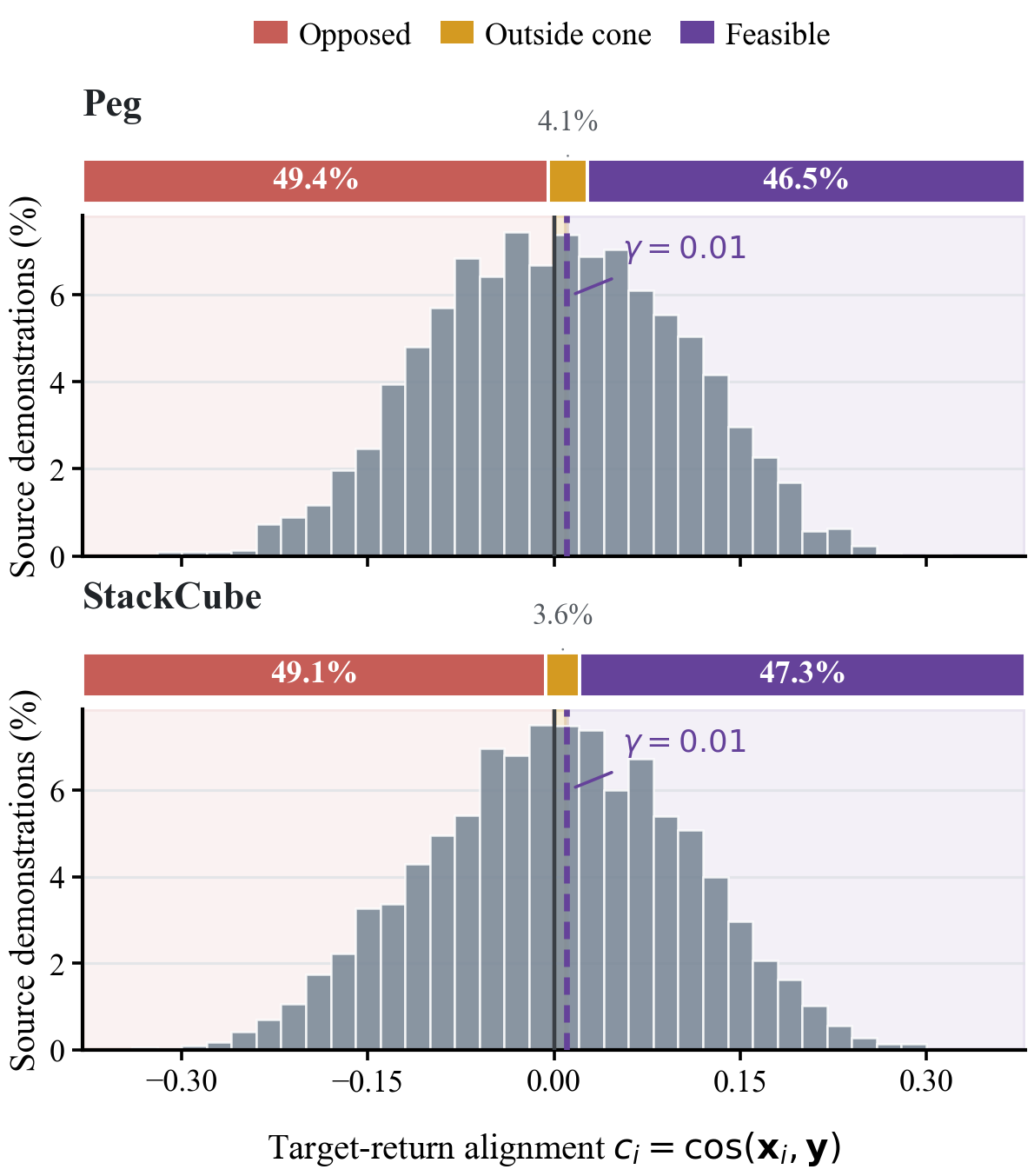}
        \captionof{figure}{Cone partitions for \texttt{Peg} and \texttt{StackCube}. Bars show region proportions; histograms show alignment $c_i$ at $\gamma=0.01$.}
        \label{fig:cone-distribution}
    \end{minipage}
\end{figure}

\section{Discussions}

We answer the questions posed in Section~\ref{sec:experiments} and discuss their implications for robot data curation.

\textbf{Q1: Curation across domain gaps.}
Our study shows that a single TUCO curation mechanism remains effective from single-simulator selection to Sim-to-Sim and Sim-to-Real co-training. This finding yields two insights into robust curation. First, demonstrations should be evaluated through their effects on target-policy closed-loop performance, since criteria based on latent structure or proxy objectives can capture domain-specific variation whose relevance changes across environments. Second, curation should be treated as set construction because a demonstration's value depends on the target data and the behaviors already represented by the selected subset. TUCO realizes both principles by measuring trajectory-level utility and set-level coverage from shared target-domain scoring rollouts, which explains its consistency as the domain gap changes.

\textbf{Q2: Active curation for sim-to-real co-training.}
To the best of our knowledge, our study provides the first systematic evidence that active curation benefits both Sim-to-Sim and Sim-to-Real co-training, and leads to a new insight: co-training should optimize which source demonstrations to mix in addition to how much source data to mix. Curated subsets outperform co-training with the full source pool,
while random subsets of the same size do not recover the same benefit. With the target demonstrations also held fixed, this comparison identifies source-data composition as a consequential choice beyond the relative amounts of source and target data. Active curation
therefore helps convert a broad simulation pool into target-relevant supervision by selecting useful and complementary demonstrations.


\section{Conclusions}
In this work, we studied simulation data curation for robot policy co-training and presented a systematic evaluation of its value in Sim-to-Real co-training. We introduced TUCO, a set-level curation framework that grounds source-data selection in target-policy closed-loop behavior. We used influence functions to trace each source demonstration to its effects on target-domain scoring rollouts, then decomposed these effects to measure trajectory-level utility and set-level coverage within a shared closed-loop space. We further developed a performance-aligned subset optimizer that combines both signals to suppress redundancy and retain complementary demonstrations. Across RoboMimic and OmniReset, TUCO achieves the best or tied-best performance among the compared methods across Single-Simulator, Sim-to-Sim, and Sim-to-Real settings.   More importantly, our experiments show that active curation improves cross-domain co-training over using the full source pool and outperforms random selection at matched data budgets, establishing the value of choosing which simulation demonstrations to mix.


\bibliography{references}
\bibliographystyle{references}

\clearpage
\appendix

\section*{Appendix Contents}
{
\hypersetup{linkcolor=TUCOSectionLink}
\setcounter{tocdepth}{2}
\startcontents[appendix]
\printcontents[appendix]{}{1}{}
}
\clearpage
\hypersetup{
  linkcolor=black,
  citecolor=TUCOCitationLink,
  urlcolor=black
}

\input{appendix}

\end{document}

%% file: math_commands.tex
\usepackage{amsmath,amsfonts,bm}

\def\eqref#1{equation~\ref{#1}}

\def\1{\bm{1}}

\DeclareMathAlphabet{\mathsfit}{\encodingdefault}{\sfdefault}{m}{sl}
\SetMathAlphabet{\mathsfit}{bold}{\encodingdefault}{\sfdefault}{bx}{n}

\DeclareMathOperator*{\argmax}{arg\,max}

%% file: authors.tex
\author[1,*]{Ning Zhu}
\author[2,3,*]{Mengfei Zhao}
\author[4,*]{Yikai Tang}
\author[5]{Zhangyujie Sun}
\author[6]{Peihao Li}
\author[7]{Dongyue Ni}
\author[2,\dagger]{Jindou Jia}
\author[2,\dagger]{Jianfei Yang}

\affiliation[1]{Stanford University}
\affiliation[2]{Nanyang Technological University}
\affiliation[3]{AXIS Robotics}
\affiliation[4]{Carnegie Mellon University}
\affiliation[5]{Fudan University}
\affiliation[6]{University of California, Berkeley}
\affiliation[7]{Shanghai Jiao Tong University}
\contribution[*]{Equal contribution}
\contribution[\dagger]{Corresponding authors}

%% file: appendix.tex
\section{Extended Related Work}
\label{app:related-work}
\subsection{Data Curation for Robot Learning}
Robot data curation methods can be categorized into three groups. \textbf{Quality-based approaches} estimate value from demonstration properties. Existing studies analyze consistency and heterogeneity~\citep{mandlekar2022whatmatters,belkhale2023dataquality}, demonstrator compatibility~\citep{gandhi2023compatible,kuhar2023discern}, and collection modality~\citep{li2025modality}. Scalable criteria use state-action dependence~\citep{hejna2025deminf} or learned task progress~\citep{zhang2026scizor,budzianowski2025opengvl}, while temporal and motion cues offer policy-free alternatives~\citep{kumar2025pacer,sojib2026efficient,kulkarni2026rinse}. Recent extensions operate on frames and action segments~\citep{yu2026frameskip,yu2026warp}, adapt temporal sampling~\citep{yang2026standardization}, and evaluate failure data~\citep{miyamoto2026failure}. Other criteria verify synthetic actions~\citep{kim2026robocurate}, focus on contact-sensitive phases~\citep{sathyanarayan2025contact}, or assess the behavioral consequences of quality metrics~\citep{bedi2026metrics}. \textbf{Utility-based approaches} condition value on target data or policy behavior. Retrieval methods match skill and state-action structure~\citep{nasiriany2023sailor,du2023behavior}, often using visual or motion features for cross-task transfer~\citep{lin2025flowretrieval,memmel2025strap}. Density correction~\citep{xie2025importance}, cue fusion~\citep{kumar2025collage}, and embodiment alignment~\citep{lin2026simdex} support heterogeneous sources. Policy-conditioned methods use rollout outcomes~\citep{chen2025demoscore} or learned datamodels~\citep{dass2026datamil}. Influence functions~\citep{koh2017influence} connect training data to imitation loss~\citep{lee2026quality} and closed-loop return~\citep{agia2025cupid,xu2026athena}, while mixture optimization estimates utility over source domains~\citep{hejna2025remix,sima2026promix}. \textbf{Diversity-based approaches} suppress representation or trajectory redundancy~\citep{zhang2026scizor,sirigiri2026faktual}. Distribution matching~\citep{feng2025tarot} and primitive composition~\citep{wu2026sieve} pursue broader coverage. This set-level view is supported by studies of behavioral multimodality~\citep{jia2024diverse}, task and embodiment diversity~\citep{shi2025diversity}, and dataset scaling~\citep{lin2025scalinglaws,saxena2025largedatasets}. Complementary analyses examine learnability~\citep{xiao2025assessment} and fragmentation~\citep{xing2025shortcut} in generalist robot data.

\subsection{Cross-Domain Robot Co-Training}
Simulation platforms provide scalable source experience for robot policy learning~\citep{mandlekar2023mimicgen,nasiriany2024robocasa}. Cross-domain co-training can be organized by how this experience is composed, aligned, and used during optimization. \textbf{Composition-oriented studies} examine simulation scale and fixed sim-to-real mixing ratios~\citep{wei2025empirical,maddukuri2025simreal}, with heterogeneous variants combining predefined modalities or embodiments~\citep{lin2026modalities,yang2026analogies}. \textbf{Alignment-oriented methods} improve source compatibility through observation-action alignment~\citep{cheng2025domainadaptation}, visual invariance~\citep{yang2025invariance}, or complementary simulation and human priors~\citep{fang2026simhuman}. Mechanistic analyses further connect co-training gains to representation alignment and domain-dependent action weighting~\citep{lei2026mechanistic}. \textbf{Utilization-oriented methods} learn weights over predefined source domains~\citep{hejna2025remix,sima2026promix}, schedule lower-quality data during diffusion training~\citep{wei2026ambient}, or combine simulation reinforcement learning with real demonstrations~\citep{shi2026rlco}. These approaches establish the importance of source composition and training strategy. However, sim-to-real co-training lacks a curation mechanism that actively selects simulation demonstrations according to their contribution to real-world policy learning.

\subsection{Active Learning}
\label{app:active-learning}
We also draw inspiration from the literature on active learning, which selects examples for annotation to improve model performance under a limited labeling budget. Diversity is an important selection principle for reducing redundant queries. \textbf{Representation-based approaches} construct representative subsets through geometric coverage~\citep{sener2018coreset,yehuda2022probcover,bae2024maxherding} or combine distribution matching with diversity during pretrained-model adaptation~\citep{xie2023activeft}. In cold-start settings, where little or no labeled data is available, typical examples from distinct clusters~\citep{hacohen2022typiclust} and self-supervised surprisal embeddings~\citep{yuan2020alps} support acquisition without a trained task-specific model. Recent methods further use self-supervised or pretrained uncertainty to guide diverse sampling and feature weighting~\citep{zhu2025csal3d,ma2025sugfw,ma2026sugfwplus}, while prediction-based and language-model-guided acquisition~\citep{guo2024deuce,bayer2026activellm} incorporate task information into cold-start selection. Benchmark evidence~\citep{zhu2025medcal} highlights the dependence of acquisition performance on both the feature extractor and the downstream task.

\textbf{Model-informed approaches} connect batch composition to learning value through joint information gain~\citep{kirsch2019batchbald} or gradient and Fisher embeddings~\citep{ash2020badge,ash2021bait}. Coverage can also be weighted by uncertainty or balanced with prediction margins as model competence changes~\citep{bae2025uherding,mishal2024dcom}. This connection to the learner extends to prediction-oriented acquisition, which targets information about deployment predictions~\citep{bickfordsmith2023epig} and expert behavior in policy fine-tuning~\citep{bagatella2025active}. These studies motivate relating subset diversity to downstream learning objectives. TUCO follows this principle by jointly measuring trajectory-level utility and set-level coverage through estimated effects on target-domain closed-loop executions.

\section{Theoretical Details}
\label{app:theory}
\label{app:method-details}

\paragraph{Roadmap.}
We derive closed-loop return influence for a target-trained base policy in Appendix~\appsectionref{app:influence-derivation} and describe practical scoring with both newly collected and recorded rollouts in Appendix~\appsectionref{app:influence-estimation}. Appendix~\appsectionref{app:subset-optimization-details} establishes the decomposition and coverage geometry, followed by a complexity analysis in Appendix~\appsectionref{app:computational-complexity}.

\subsection{Closed-Loop Performance Influence}
\label{app:influence-derivation}

We justify the return-influence interpretation of the rollout-wise profiles and utility scores in Section~\appsectionref{sec:closed-loop-attribution} for a base policy trained on the target data. We use the standard influence-function result~\citep{koh2017influence} and CUPID's closed-loop attribution argument~\citep{agia2025cupid} in our source-to-target notation.

\paragraph{Conditions.}
We assume that $\theta_0$ is an interior local minimizer of $\mathcal L_{\mathrm t}$ with nonsingular Hessian $\mathbf H_0$, and that both $\mathcal L_{\mathrm t}$ and $\ell_\xi(\xi_i^{\mathrm s}\mid\theta)$ are twice continuously differentiable near $\theta_0$. The target environment and return function are independent of $\theta$. The rollout density is differentiable with locally fixed support, and differentiation can be exchanged with integration. The return and return-weighted rollout score are integrable.

\begin{theorem}[Return-Weighted Performance Influence]
\label{thm:return-weighted-influence}
Under the conditions above, let $\theta_{\epsilon,i}$ be the local minimizer of $\mathcal L_{\mathrm t}(\theta)+\epsilon\ell_\xi(\xi_i^{\mathrm s}\mid\theta)$ continuing from $\theta_{0,i}=\theta_0$. Define $\mathbf g_i:=\nabla_\theta\ell_\xi(\xi_i^{\mathrm s}\mid\theta_0)$ and $\mathbf h_r:=\left.\nabla_\theta\log p_{\mathrm t}(\tau_r\mid\pi_\theta)\right|_{\theta=\theta_0}$. For a target rollout in the support of the base policy,
\begin{equation}
\begin{aligned}
X_{ri}
&:=\left.\partial_\epsilon\log p_{\mathrm t}(\tau_r\mid\pi_{\theta_{\epsilon,i}})\right|_{\epsilon=0}\\
&=-\mathbf h_r^\top\mathbf H_0^{-1}\mathbf g_i.
\end{aligned}
\label{eq:app-rollout-wise-influence}
\end{equation}
With $y_r:=R_{\mathrm t}(\tau_r)$, its return-weighted expectation gives
\begin{equation}
\begin{aligned}
\mathcal I_i^{\mathrm{perf}}
&:=\left.\partial_\epsilon J_{\mathrm t}(\theta_{\epsilon,i})\right|_{\epsilon=0}\\
&=\mathbb E_{\tau_r\sim p_{\mathrm t}(\cdot\mid\pi_{\theta_0})}[y_rX_{ri}].
\end{aligned}
\label{eq:app-return-weighted-influence}
\end{equation}
\end{theorem}

\begin{proof}
Differentiating the stationarity condition for the perturbed objective gives the standard parameter response~\citep{koh2017influence}:
\begin{equation}
\left.\partial_\epsilon\theta_{\epsilon,i}\right|_{\epsilon=0}
=-\mathbf H_0^{-1}\mathbf g_i.
\label{eq:app-parameter-response}
\end{equation}
The environment factors in the rollout likelihood are independent of $\theta$, so
\begin{equation}
\mathbf h_r
=\sum_{h=0}^{H_r-1}
\left.\nabla_\theta\log\pi_\theta(a_{r,h}\mid o_{r,h})\right|_{\theta=\theta_0}.
\label{eq:app-rollout-score}
\end{equation}
Applying the chain rule to the log likelihood of a fixed rollout then yields Equation~\appequationref{eq:app-rollout-wise-influence}.

For expected return, the likelihood-ratio identity gives
$\nabla_\theta J_{\mathrm t}(\theta_0)=\mathbb E_{\tau_r}[y_r\mathbf h_r]$.
Combining it with the parameter response yields
\begin{equation}
\begin{aligned}
\mathcal I_i^{\mathrm{perf}}
&=\nabla_\theta J_{\mathrm t}(\theta_0)^\top
\left.\partial_\epsilon\theta_{\epsilon,i}\right|_{\epsilon=0}\\
&=-\mathbb E_{\tau_r}[y_r\mathbf h_r]^\top\mathbf H_0^{-1}\mathbf g_i\\
&=\mathbb E_{\tau_r}[y_rX_{ri}],
\end{aligned}
\label{eq:app-performance-chain-rule}
\end{equation}
where $\mathbf H_0$ and $\mathbf g_i$ are fixed under the rollout expectation.
\end{proof}


By Theorem~\apptheoremref{thm:return-weighted-influence}, $q_i/M=M^{-1}\sum_{r=1}^{M}y_rX_{ri}$ is an unbiased estimator of local return influence when the theorem's conditions hold, profile entries are exact, and the scoring rollouts are independently sampled from $p_{\mathrm t}(\cdot\mid\pi_{\theta_0})$. With returns $+1$ for success and $-1$ for failure, $J_{\mathrm t}=2P_{\mathrm{success}}-1$, so $q_i/(2M)$ estimates the corresponding first-order change in success probability. The next subsection describes the practical estimator and the interpretation of the success-only feedback used in Sim-to-Real.

\subsection{Practical Influence Estimation}
\label{app:influence-estimation}

We follow CUPID~\citep{agia2025cupid} to estimate influence profiles using projected loss gradients. We evaluate all gradients at the fixed base policy $\theta_0$ and use the same target-domain scoring rollouts for every source candidate. In the simulated settings, the base policy is trained on target demonstrations. For Sim-to-Real, it is pretrained on simulation data and supplies recorded successful real-world rollouts.

\paragraph{Trajectory aggregation.}
Let $\mathbf e_{r,h}:=\nabla_\theta\ell(o_{r,h},a_{r,h}\mid\theta_0)$ and $\mathbf g_{i,k}:=\nabla_\theta\ell(o_{i,k}^{\mathrm s},a_{i,k}^{\mathrm s}\mid\theta_0)$ denote target-rollout and source-demonstration loss gradients. For a negative log-likelihood loss, $\mathbf h_r=-\sum_h\mathbf e_{r,h}$ and $\mathbf g_i=\sum_k\mathbf g_{i,k}$. Thus,
\begin{equation}
X_{ri}
=\sum_{h=0}^{H_r-1}\sum_{k=0}^{T_i-1}
\mathbf e_{r,h}^\top\mathbf H_0^{-1}\mathbf g_{i,k}.
\label{eq:app-action-factorization}
\end{equation}
For a deterministic MLP, we use mean-squared behavior-cloning loss as a surrogate. This loss corresponds to Gaussian negative log likelihood up to a positive scale and an additive constant when the action covariance is fixed and isotropic.

\paragraph{Projected estimator.}
Let $\mathbf P\in\mathbb R^{d_\theta\times d}$ be a shared random projection from the $d_\theta$ policy parameters to $d$ dimensions. We collect the projected gradients of source training examples as rows of $\widetilde{\mathbf G}\in\mathbb R^{L_{\mathrm s}\times d}$, where $L_{\mathrm s}$ counts individual transitions for the state policy or observation--action windows for Diffusion Policy, and define

\begin{equation}
\widetilde{\mathbf C}
:=\widetilde{\mathbf G}^\top\widetilde{\mathbf G}+\eta\mathbf I_d,
\qquad \eta\geq0.
\label{eq:app-projected-curvature}
\end{equation}
We require $\widetilde{\mathbf C}$ to be invertible. This source-gradient Gram matrix provides the curvature approximation used for scoring at the frozen base policy. In the exact target-trained derivation, $\mathbf H_0$ is the Hessian of the target training loss. With $\widetilde{\mathbf e}_r:=\sum_h\mathbf P^\top\mathbf e_{r,h}$ and $\widetilde{\mathbf g}_i:=\sum_k\mathbf P^\top\mathbf g_{i,k}$, the estimated profile entry is
\begin{equation}
\widehat X_{ri}
=\widetilde{\mathbf e}_r^\top\widetilde{\mathbf C}^{-1}\widetilde{\mathbf g}_i.
\label{eq:app-projected-influence}
\end{equation}
Bilinearity allows trajectory gradients to be summed before scoring, preserving the frame-pair aggregation in Equation~\appequationref{eq:app-action-factorization} under this curvature approximation. We use these estimated profiles as $\mathbf X$ in TUCO's decomposition and subset selection.

\subsection{Influence Decomposition and Subset Geometry}
\label{app:subset-optimization-details}
\label{app:subset-computation}

We establish the utility-preserving decomposition in Section~\appsectionref{sec:influence-decomposition} and the volume interpretation and marginal gain of the coverage criterion in Section~\appsectionref{sec:subset-optimization}. The following identities hold for any fixed influence profiles and $\mathbf y\neq\mathbf0$, including the estimated profiles constructed from recorded successful rollouts.

\paragraph{Decomposition and performance alignment.}
Recall $\mathbf x_i=u_i\widehat{\mathbf y}+\mathbf z_i$ from Equation~\appequationref{eq:influence-decomposition}. Since $\mathbf z_i$ is orthogonal to the unit vector $\widehat{\mathbf y}$,
\begin{equation}
\begin{aligned}
\widehat{\mathbf y}^\top\mathbf z_i&=0,\\
q_i&=\|\mathbf y\|_2u_i,\\
\|\mathbf x_i\|_2^2&=u_i^2+\|\mathbf z_i\|_2^2.
\end{aligned}
\label{eq:app-decomposition-identities}
\end{equation}
The decomposition therefore preserves utility. When $q_i=q_j$, we have $\mathbf x_i-\mathbf x_j=\mathbf z_i-\mathbf z_j$, so the residuals retain all differences between profiles with equal utility. For $q_i>0$, the alignment constraint in Equation~\appequationref{eq:feasible-cone} becomes
\begin{equation}
\frac{u_i}{\sqrt{u_i^2+\|\mathbf z_i\|_2^2}}\geq\gamma.
\label{eq:app-cone-decomposition}
\end{equation}
For $\gamma\in(0,1)$, this is equivalent to $\|\mathbf z_i\|_2\leq u_i\sqrt{\gamma^{-2}-1}$. The cone thus bounds residual magnitude relative to the positive return-aligned component before optimizing directional coverage. At $\gamma=0$, it reduces to the positive-utility half-space.

For success-only feedback, $\mathbf y=\mathbf1\in\mathbb R^M$ and $\widehat{\mathbf y}=\mathbf1/\sqrt M$. The decomposition becomes
\begin{equation}
q_i=\sum_{r=1}^{M}X_{ri},\qquad
u_i=\frac{q_i}{\sqrt M},\qquad
\mathbf z_i=\mathbf x_i-\frac{q_i}{M}\mathbf1.
\label{eq:app-success-only-decomposition}
\end{equation}
Thus, utility measures aggregate support across the successful scoring rollouts, while the residual records variation around that profile's mean influence. Even with identical success labels, the residuals retain differences across executions that can support complementary coverage. The same cone constraint and subset optimizer apply to these profiles.

\paragraph{Spanning-volume interpretation.}
For a nonempty feasible pool, use the positive weights $w_i=q_i/\kappa_{\mathcal F}$ from the main text. Normalize nonzero residuals to $\mathbf v_i=\mathbf z_i/\|\mathbf z_i\|_2$ and set $\mathbf v_i=\mathbf0$ otherwise. For $S\subseteq\mathcal F_\gamma$, define
\begin{equation}
\begin{aligned}
\mathbf V_S&:=[\sqrt{\rho w_i}\mathbf v_i]_{i\in S},\\
\mathbf M(S)&:=\mathbf I_M+\mathbf V_S\mathbf V_S^\top.
\end{aligned}
\end{equation}
Because $\mathbf M(S)\succeq\mathbf I_M$, the coverage score $G(S)=\log\det\mathbf M(S)$ is well defined for every subset, including the empty set. Sylvester's determinant identity and the principal-minor expansion give
\begin{equation}
\begin{aligned}
\exp G(S)
&=\det\!\left(\mathbf I_{|S|}+\mathbf V_S^\top\mathbf V_S\right)\\
&=\sum_{T\subseteq S}\rho^{|T|}
\left(\prod_{i\in T}w_i\right)
\det\!\left([\mathbf v_i^\top\mathbf v_j]_{i,j\in T}\right).
\end{aligned}
\label{eq:app-spanning-volume}
\end{equation}
The empty-set term is one. Each Gram determinant equals the squared parallelotope volume spanned by its residual vectors and vanishes for linearly dependent vectors. Thus, $\exp G(S)$ aggregates utility-weighted squared spanning volumes in the subspace orthogonal to $\mathbf y$. This gives a precise interpretation of the coverage intuition in the main text. The outer product $\mathbf v_i\mathbf v_i^\top$ treats opposite residual directions as the same span. A zero residual contributes no coverage while remaining eligible if it satisfies the cone constraint.

\paragraph{Marginal coverage.}
Adding $i\in\mathcal F_\gamma\setminus S$ updates the coverage matrix by $\rho w_i\mathbf v_i\mathbf v_i^\top$. Applying the matrix determinant lemma gives
\begin{equation}
\begin{aligned}
\Delta_i(S)
&=\log\frac{\det(\mathbf M(S)+\rho w_i\mathbf v_i\mathbf v_i^\top)}
{\det\mathbf M(S)}\\
&=\log\!\left(1+\rho w_i\mathbf v_i^\top\mathbf M(S)^{-1}\mathbf v_i\right).
\end{aligned}
\label{eq:app-logdet-marginal}
\end{equation}
This is the marginal gain used in Algorithm~\ref{alg:subset-selection}. To see how it accounts for existing coverage, write $\mathbf M(S)=\mathbf Q\operatorname{diag}(d_1,\ldots,d_M)\mathbf Q^\top$ in an orthonormal eigenbasis. Then
\begin{equation}
\mathbf v_i^\top\mathbf M(S)^{-1}\mathbf v_i
=\sum_{k=1}^{M}\frac{(\mathbf Q_{:,k}^\top\mathbf v_i)^2}{d_k},
\qquad d_k\geq1.
\end{equation}
A larger $d_k$ indicates greater weighted support along that direction in $S$. Its inverse reduces the contribution of an already covered direction, favoring additional coverage for candidates with comparable utility. For $S\subseteq T\subseteq\mathcal F_\gamma$, we have $\mathbf M(T)\succeq\mathbf M(S)\succ0$, hence $\mathbf M(T)^{-1}\preceq\mathbf M(S)^{-1}$. Equation~\appequationref{eq:app-logdet-marginal} therefore implies $0\leq\Delta_i(T)\leq\Delta_i(S)$ for $i\in\mathcal F_\gamma\setminus T$.

\paragraph{Selection and budget completion.}
Equation~\appequationref{eq:greedy-rule} combines fixed utility scores with marginal coverage that changes as $S$ grows. Utility is standardized over the feasible pool, while coverage gains are standardized over the remaining candidates at each iteration. A zero variance gives zero standardized scores. If $K>|\mathcal F_\gamma|$, all feasible candidates are retained and the remaining slots are filled by the highest-utility unselected trajectories. An empty feasible pool therefore yields utility-only selection. Once the budget includes the entire feasible pool, the selected set is independent of $\lambda$.

\subsection{Computational Complexity}
\label{app:computational-complexity}

Given $\mathbf X\in\mathbb R^{M\times N}$, decomposition and cone filtering take $O(NM)$ time. Let $F:=|\mathcal F_\gamma|$ and $K_{\mathcal F}:=\min(K,F)$. Evaluating dense quadratic forms across $K_{\mathcal F}$ greedy iterations costs $O(K_{\mathcal F}FM^2)$, with $O(M^2)$ per rank-one inverse update. Utility sorting for budget completion costs $O(N\log N)$, and memory usage is $O(NM+M^2)$. Computing the complete feasible ordering once for reuse across budgets takes $O(F^2M^2)$ time. These costs cover subset selection after influence estimation.

\section{Additional Experimental Details}
\label{app:experimental-details}

\subsection{Tasks and Datasets}
\label{app:tasks-datasets}

We adopt the single-simulator protocol of CUPID~\citep{agia2025cupid} on the low-dimensional Lift, Square, and Transport tasks from RoboMimic~\citep{mandlekar2022whatmatters}. We follow its data splits, validation protocol, curation budgets, and task-specific training schedules. Filtering is evaluated on all three tasks, while selection is evaluated on Square and Transport. 

The Sim-to-Sim and Sim-to-Real study uses \texttt{Peg}, \texttt{StackCube}, and \texttt{CupCake} from OmniReset~\citep{yin2026omnireset}. For Sim-to-Sim, MuJoCo serves as the target domain and IsaacSim supplies the source demonstrations. We retain the complete target set and curate source trajectories at the budgets defined in Section~\appsectionref{sec:experiments}. The resulting subset is combined with the target demonstrations, and policy performance is evaluated in the target environment. Figure~\ref{fig:app-sim-task-trajectories} illustrates the simulated task executions.

\begin{figure}[tbp]
    \centering
    \includegraphics[width=\textwidth]{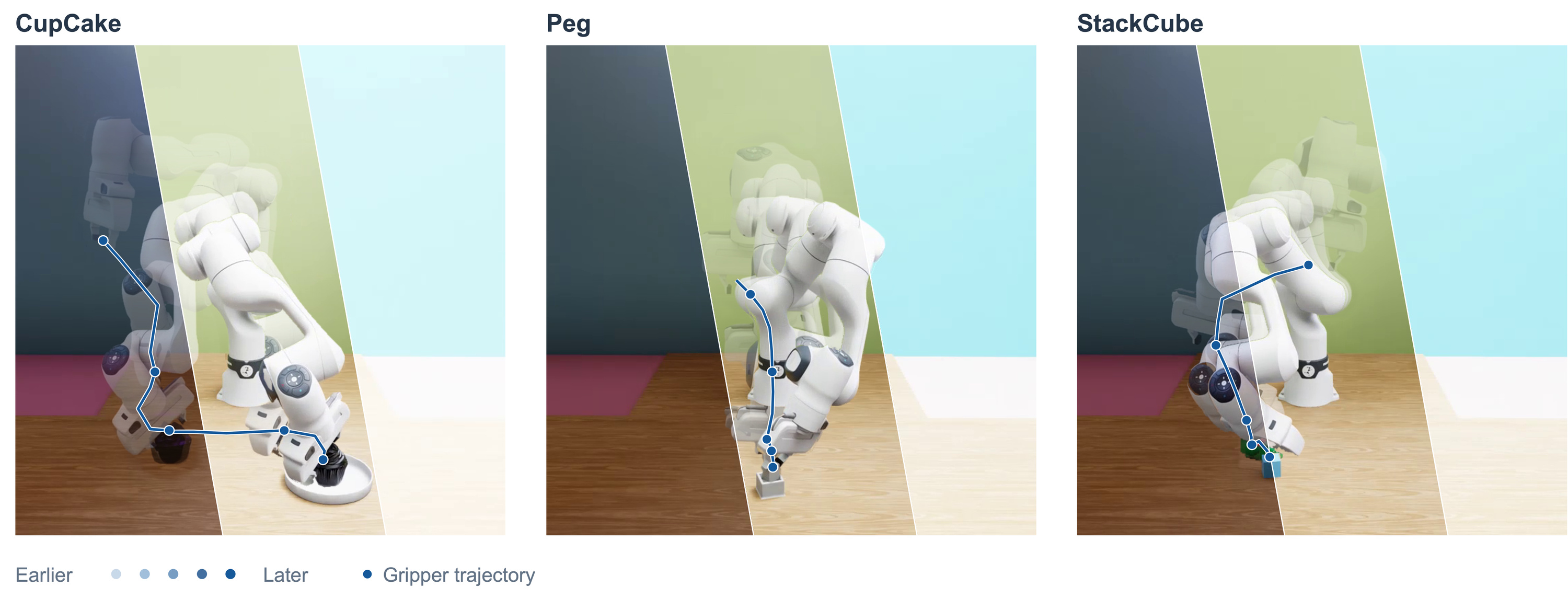}
    \caption{Example simulated executions of \texttt{CupCake}, \texttt{Peg}, and \texttt{StackCube}. Overlaid frames show successive stages of each task, with blue paths tracing gripper motion.}
    \label{fig:app-sim-task-trajectories}
\end{figure}

For each task of the Sim-to-Real study, we select 600 demonstrations from a source pool of 1,200 simulation demonstrations. The selected source demonstrations are then combined with real-world demonstrations for co-training. During mixed-domain training, real and simulation samples account for 20\% and 80\% of the sampling distribution, respectively. Real-world training samples are aligned with individual policy decisions by pairing the observations used for a decision with the action commands actually issued. We pad these samples to the policy prediction horizon and exclude padded action positions from the training loss. 

\subsection{Simulation Policy Training Details}
\label{app:policy-training}

The single-simulator experiments use the convolutional Diffusion Policy configuration released with CUPID~\citep{agia2025cupid}. The policy has a prediction horizon of 16 with two observation steps and eight action steps. Each curated subset is trained from scratch. Methods compared at the same task and budget share the policy configuration, training schedule, and normalization procedure.

The Sim-to-Sim experiments use a state-based behavior-cloning policy that takes two consecutive 200-dimensional observations. The resulting 400-dimensional input is processed by an MLP with widths $400\rightarrow1024\rightarrow512\rightarrow512\rightarrow512\rightarrow7$ and ReLU activations. We minimize mean-squared action error for 50,000 gradient steps using AdamW. The batch size is 256 and the learning rate is $10^{-4}$. AdamW uses betas $(0.95,0.999)$ and weight decay $10^{-6}$. The learning rate is warmed up for 500 steps and then follows cosine decay. Gradients are clipped to unit norm. Evaluation uses exponential moving average weights with power $0.75$ and maximum decay $0.9999$. Every policy is trained from scratch on all target demonstrations and the selected source trajectories. Training examples are sampled uniformly from the combined data, while the normalization statistics are computed from the target set.

\subsection{Real-World Policy Training Details}

\paragraph{Base policy and scoring rollouts.}
We obtain the base-policy weights by training a task-specific image-conditioned Diffusion Policy on simulated demonstrations using behavior cloning with a diffusion denoising objective. We deploy this simulation-pretrained policy on the physical robot and retain successful executions confirmed by the operator. These recorded trajectories supply both the target demonstrations for co-training and the target-domain scoring rollouts for curation. The base-policy weights remain frozen when computing source and target attribution gradients. After selection, we initialize a new policy and train it on all target demonstrations together with the curated source subset. Final evaluation uses separate real-world executions.

\paragraph{Robot platform and cameras.}
We use a Franka Research 3 robot equipped with a Franka Gripper. The camera setup consists of an Intel RealSense D405 mounted at the wrist, an Intel RealSense D435 providing the side view, and an Intel RealSense L515 providing the head view. The policy receives RGB observations from all three cameras together with proprioceptive observations consisting of the seven arm joint positions and the gripper width.

\paragraph{Task objects and environment alignment.}
The task objects are fabricated using a Bambu Lab P1S 3D printer. We use corresponding objects in simulation and the physical workspace, calibrate the cameras, and match the object-placement regions across domains. This alignment reduces differences arising from object geometry, camera placement, and workspace configuration while retaining the dynamics and visual differences encountered during physical execution. Figure~\ref{fig:app-real-task-trajectories} shows the corresponding task executions on the physical robot.

\begin{figure}[t!]
    \centering
    \includegraphics[width=\textwidth]{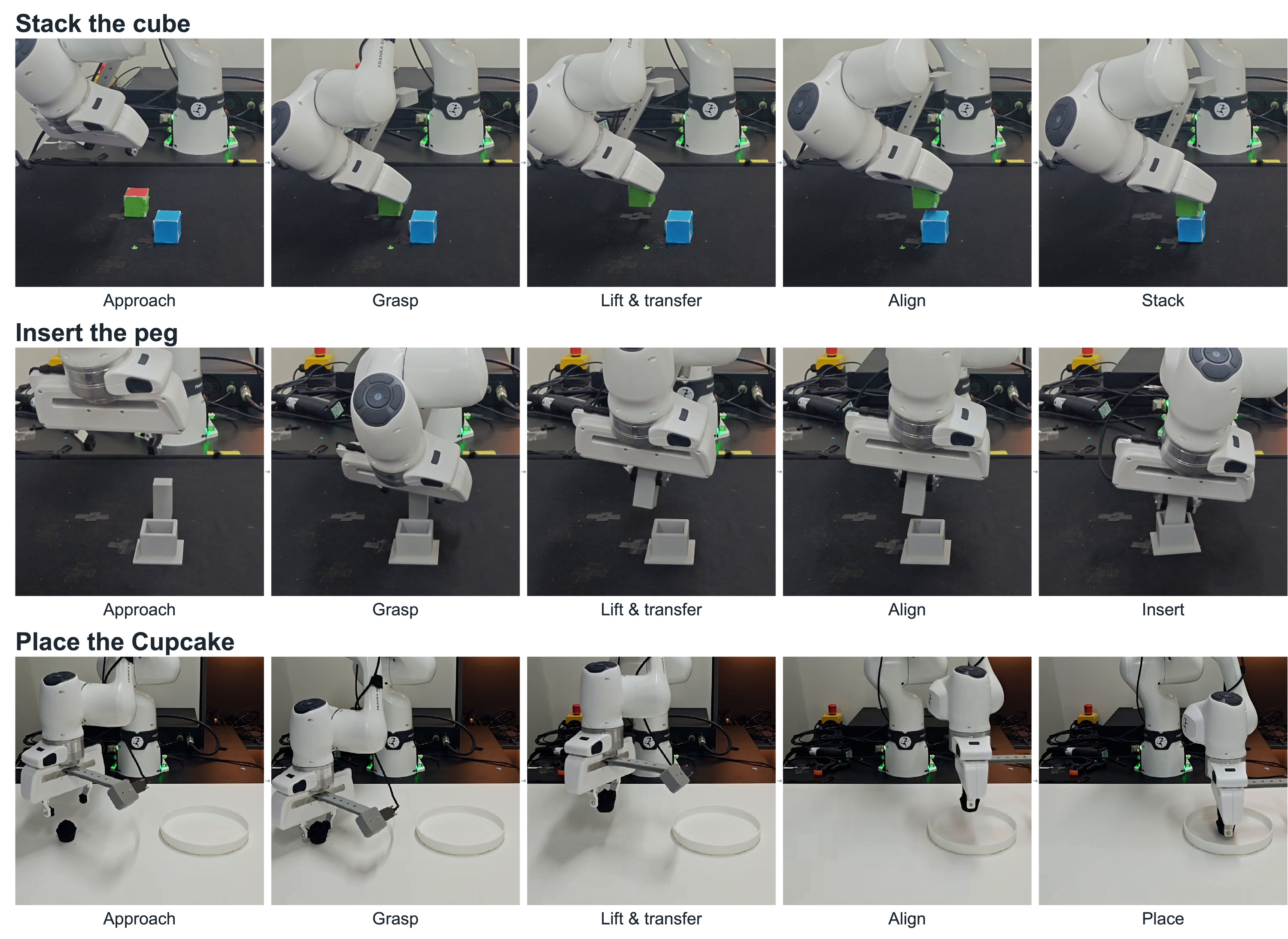}
    \caption{Example real-world executions of \texttt{StackCube}, \texttt{Peg}, and \texttt{CupCake}, shown from top to bottom. Frames progress from left to right through grasping, transfer, and final placement or insertion.}
    \label{fig:app-real-task-trajectories}
\end{figure}

\paragraph{Policy architecture.}
We use an image-conditioned Diffusion Policy with a separate ResNet-18 encoder for each camera. The visual features are concatenated with the proprioceptive observations from two consecutive time steps and supplied as global conditioning to a temporal U-Net. Input images are resized to $84 \times 84$ pixels and randomly cropped to $76 \times 76$ pixels during training. The U-Net predicts a sequence of eight-dimensional actions, with seven arm joint commands and one gripper command per time step. We use a prediction horizon of 16 and a DDPM noise scheduler with 100 diffusion steps. The image encoders and action network are trained jointly.

\paragraph{Optimization.}
We train the policies from scratch for 150 epochs using AdamW with a batch size of 64. The learning rate is warmed up for 500 optimization steps and then follows a cosine decay schedule. We minimize the mean-squared error between the predicted and sampled diffusion noise over valid action positions. We maintain an exponential moving average of the policy parameters with a maximum decay of 0.9999. The training configuration is summarized in Table~\ref{tab:sim2real-training-config}.

\begin{table}[t!]
    \centering
    \caption{Policy and training configuration for sim-to-real co-training.}
    \label{tab:sim2real-training-config}
    \small
    \setlength{\tabcolsep}{8pt}
    \renewcommand{\arraystretch}{1.1}
    \begin{tabular}{ll}
        \toprule
        Parameter & Value \\
        \midrule
        Observation history & 2 steps \\
        Prediction horizon & 16 steps \\
        Image resolution & $84 \times 84$ \\
        Training crop & $76 \times 76$ \\
        Visual encoder & ResNet-18 per camera \\
        U-Net channel widths & 256, 512, 1024 \\
        Diffusion steps & 100 \\
        Noise schedule & Cosine \\
        Optimizer & AdamW \\
        Learning rate & $10^{-4}$ \\
        AdamW betas & $(0.95, 0.999)$ \\
        Weight decay & $10^{-6}$ \\
        Learning-rate schedule & Cosine decay \\
        Warmup & 500 optimization steps \\
        Batch size & 64 \\
        Training epochs & 150 \\
        Real / simulation sampling & 20\% / 80\% \\
        Maximum EMA decay & 0.9999 \\
        \bottomrule
    \end{tabular}
\end{table}

\subsection{Baseline Implementations}
\label{app:baseline-implementations}

All baselines curate complete trajectories from the same candidate pool. Filtering removes the lowest-ranked demonstrations from the initial training split. Selection adds the highest-ranked source trajectories to the fixed target set. Random samples trajectories without replacement. Target-only and full-pool controls use the same policy and optimization protocol as the curated subsets. Scalar-score methods share a fixed trajectory ordering across budgets. Budget-dependent methods receive the requested cardinality as part of their subset objective.

In Sim-to-Real, policy-conditioned scoring uses the shared simulation-pretrained base policy, and target-conditioned methods receive the same recorded successful real-world data. The selected source subsets are combined with this fixed target set for training from scratch.

\paragraph{DemInf.}
We use the released DemInf implementation for single-simulator filtering~\citep{hejna2025deminf}. For Sim-to-Sim, the identity encoder is applied to the state and action. The two feature groups are standardized separately over the source pool. We estimate their local dependence with the KSG estimator using Chebyshev distance and $k=5$. Frame-level estimates are averaged within each trajectory, and trajectories with larger dependence are selected.

\paragraph{Demo-SCORE.}
The single-simulator implementation follows the adaptation in CUPID~\citep{chen2025demoscore,agia2025cupid}. For Sim-to-Sim, each target-rollout frame inherits its rollout success label. Actions are standardized using the target-rollout statistics. We train a $7\rightarrow8\rightarrow8\rightarrow1$ classifier for 200 epochs using dropout $0.3$ and weighted binary cross entropy. The optimizer is AdamW with learning rate $10^{-4}$ and weight decay $0.1$. A source trajectory is scored by the mean predicted success probability of its actions.

\paragraph{Success Similarity.}
The single-simulator implementation follows CUPID and uses low-dimensional state observations~\citep{agia2025cupid}. For Sim-to-Sim, source states and target-rollout states are jointly min-max normalized. We compute the mean pairwise Euclidean distance between every source trajectory and each successful target rollout. The negative distance averaged over successful rollouts defines the trajectory score. For Sim-to-Real, we encode image and proprioceptive observations with the frozen Diffusion Policy observation encoder and average the features over the sampled windows of each candidate or real rollout. We rank candidates by their mean cosine similarity to the successful real-rollout representations.

\paragraph{CUPID.}
We follow the closed-loop attribution procedure released with CUPID~\citep{agia2025cupid}. The single-simulator implementation evaluates each transition over 64 diffusion-timestep and noise samples before applying a 4,000-dimensional TRAK projection. For Sim-to-Sim, we compute full-policy gradients of the behavior-cloning loss and use the same projection dimension. In these simulated settings, frame-level influences are summed within source demonstrations and target rollouts. The resulting utility adds influence on successful rollouts and subtracts influence on failed rollouts. For Sim-to-Real, we average projected gradients within each source candidate and real rollout, then rank candidates by their mean estimated influence across the recorded successful rollouts. All real scoring rollouts receive positive return labels.

\paragraph{Quality over Quantity.}
We preserve the normalized policy-gradient similarity of Quality over Quantity~\citep{lee2026quality}. Diffusion Policy gradients are computed from the final denoising block and compressed to 2,048 dimensions using OPORP. We sample at most 32 training windows from each demonstration. Each window receives its maximum similarity to a validation window, followed by averaging within its trajectory. For Sim-to-Sim, we use the normalized Gaussian log-likelihood gradient of the action-head weight and bias. The gradients are compressed to 1,024 dimensions. Each source frame is matched to its most similar target frame from a reference set capped at 4,096 frames. These similarities are then averaged within the source trajectory.

\paragraph{PSD.}
We implement the spectral quality metric of~\citet{sojib2026efficient} using the end-effector Cartesian position. A full discrete Fourier transform is applied independently to each coordinate without detrending or removing its constant component. Spectral power is summed across coordinates and frequencies. Trajectories with lower total power receive higher priority.

\paragraph{DataMIL.}
For single-simulator Diffusion Policy, we use a one-step adaptation of the DataMIL metagradient estimator on the final denoising block~\citep{dass2026datamil}. Each source trajectory receives a differentiable weight, and its score is obtained by differentiating target validation loss through the inner update. For Sim-to-Sim, all features use target-domain normalization and every fourth frame is retained for attribution. We reserve 100 target trajectories for meta-validation. The pretrained policy trunk is frozen, while the action head is differentiably unrolled for 200 updates with learning rate $10^{-3}$ and batch size 256. Each update contains equal numbers of target and weighted source examples. The negative derivative of target validation error with respect to each trajectory weight gives its score. All target trajectories are restored for final co-training. For Sim-to-Real, we use a one-step gradient-alignment adaptation. We compute full-policy gradients of the squared denoiser-output objective and project them to 2,048 dimensions, averaging over sampled windows within each source candidate and real rollout. The reference gradient is the mean across successful real rollouts. We rank source candidates by their gradient inner product with this reference and retain all target demonstrations for final co-training.

\paragraph{FAKTUAL.}
For simulated experiments, we use the exact level-two trajectory signature and normalized signature Gram matrix of FAKTUAL~\citep{sirigiri2026faktual}. In the single-simulator setting, each path contains the object state, end-effector pose, gripper state, and raw action. All path features are standardized and augmented with normalized time, then a zero base point is prepended. We follow the entropy and determinant selection procedure with entropy fraction $0.5$. Its stochastic-greedy tolerance is $0.1$ and determinant regularization is $10^{-8}$. For Sim-to-Sim, the state-action path retains the most recent value of each history-based state term. We use determinant-greedy selection with the regularized objective $\log\det(\mathbf K+10^{-4}\mathbf I)$. The selected prefixes provide the subsets at different budgets. For Sim-to-Real, we adapt the determinant-based selector to the projected full-policy gradient representations used by the DataMIL adaptation above. We normalize these source representations to construct their Gram matrix and greedily select candidates by the log-determinant gain with regularization $10^{-4}$. This adaptation measures diversity in policy-gradient space and uses only source representations for subset selection.

\paragraph{SIEVE.}
We preserve the segmentation, primitive discovery, and pattern-aware selection stages of SIEVE~\citep{wu2026sieve}. A gripper-command transition defines a segment boundary after the new state persists for five frames. Each segment is represented by standardized state-action features at its start, middle, and end. The single-simulator implementation uses the complete low-dimensional observation, while Sim-to-Sim uses the most recent value of each state term and the current action. The descriptors are projected to at most 256 principal components. We search 20 candidate primitive counts on a 10\% trajectory subset, using minimum average cluster sizes of 20 and 1,000 for the two settings. Each trajectory is represented as an ordered primitive pattern. Budgets are allocated by the marginal coverage of primitives and transitions, while cosine-medoid distance orders demonstrations that share a pattern.

\paragraph{TAROT.}
We adapt the gradient transport representation of TAROT~\citep{feng2025tarot}. Diffusion Policy trajectories use the mean DDPM-loss gradient of the final denoising block. Sim-to-Sim trajectories use the mean action-head gradient over the trajectory. Both representations are compressed to 2,048 dimensions and jointly centered. They are whitened with covariance regularization $10^{-5}$ and normalized to unit length. Selection expands source neighborhoods around target trajectories. An entropic optimal-transport solve with regularization $0.01$ and 200 iterations resolves the final neighborhood when it exceeds the remaining budget. We disable data-repetition weights so that each method changes only the selected source subset.

\subsection{Evaluation Protocol}
\label{app:evaluation-protocol}
Final evaluation rollouts are collected from the retrained policies and are separate from the scoring rollouts used for curation. In the simulated OmniReset tasks, \texttt{Peg} succeeds when a 160-step rollout reaches position error below $0.01$ m and roll-plus-pitch error below $0.1$ rad. \texttt{StackCube} requires stable pose and object contact after release and settling, with position error below $0.005$ m and roll-plus-pitch error below $0.025$ rad. \texttt{CupCake} requires the same pose thresholds to hold for five consecutive steps within a 160-step rollout.

For each real-world task and method, we evaluate the checkpoint from the final training epoch over 40 trials. The lighting-shift experiment evaluates the same checkpoint over 20 additional trials.

\subsection{Influence-Space Visualization}
\label{app:visualization-results}

To examine how different curation criteria balance utility and coverage, we visualize the \texttt{StackCube} subsets selected at a $10\%$ source budget. Figure~\ref{fig:stackcube-influence-space} projects the orthogonal influence components onto two principal components and colors selected demonstrations by utility percentile. CUPID concentrates high-utility trajectories in a narrow region, indicating repeated closed-loop influence patterns. Diversity-based methods cover a broader region but include more low-utility trajectories because their selection spaces are not grounded in target-policy performance. TUCO retains high-utility candidates while spanning a broader range of influence directions, illustrating how closed-loop utility and marginal coverage jointly shape the selected subset.

\begin{figure}[htbp]
    \centering
    \includegraphics[width=\textwidth]{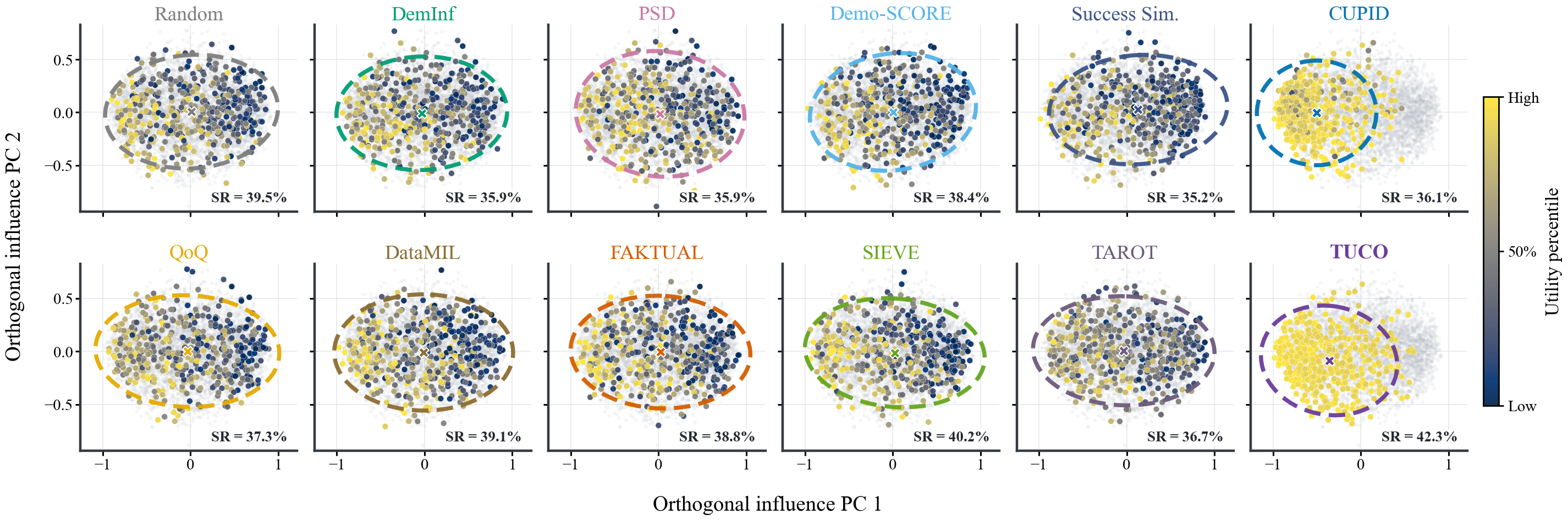}
    \caption{Influence-space visualization of \texttt{StackCube} subsets selected at a $10\%$ source budget. Gray points denote the source pool and colored points denote selected demonstrations by utility percentile. Ellipses summarize the PCA-projected spread of each subset.}
    \label{fig:stackcube-influence-space}
\end{figure}

\section{Extended Experiments}
\subsection{Hyperparameter Sensitivity}
\label{app:hyperparameter-sensitivity}

We study how the utility--coverage weight $\lambda$ and the cone threshold $\gamma$ affect subset selection. To isolate their effects, we vary $\lambda$ with $\gamma=0.01$ and vary $\gamma$ with $\lambda=0.2$.

\paragraph{Effect of $\lambda$.}
Table~\ref{tab:lambda-sensitivity-single-sim} presents the complete single-simulator results. Averaged across all 25 task and budget configurations, $\lambda\in\{0.1,0.2,0.5\}$ achieves success rates of $59.76\%$, $61.10\%$, and $60.18\%$, respectively. The configuration $\lambda=0.2$ obtains the highest or tied-highest result in 17 configurations and performs consistently in both filtering and selection. The best value still varies in individual cases. Square selection at the $75\%$ budget changes by up to $9.8\%$ across the tested values, while Transport selection at the $10\%$ budget favors $\lambda=0.5$. The preferred balance between utility and coverage can therefore depend on the task and subset size.

\begin{table}[H]
    \centering
    \caption{Sensitivity to $\lambda$ in the single-simulator setting with $\gamma=0.01$. Entries are target-domain success rates reported as mean $\pm$ standard deviation. Bold parameter values denote the default configuration.}
    \label{tab:lambda-sensitivity-single-sim}
    \footnotesize
    \setlength{\tabcolsep}{3.5pt}
    \renewcommand{\arraystretch}{1.10}
    \begin{tabular}{lllccccc}
        \toprule
        Protocol & Task & $\lambda$ & $10\%$ & $25\%$ & $50\%$ & $75\%$ & $90\%$ \\
        \midrule
        Filtering & Lift & $0.1$ & $1.000{\pm}0.000$ & $1.000{\pm}0.000$ & $1.000{\pm}0.000$ & $0.986{\pm}0.013$ & $0.960{\pm}0.023$ \\
                  &      & $\mathbf{0.2}$ & $1.000{\pm}0.000$ & $1.000{\pm}0.000$ & $1.000{\pm}0.000$ & $1.000{\pm}0.000$ & $0.988{\pm}0.016$ \\
                  &      & $0.5$ & $1.000{\pm}0.000$ & $1.000{\pm}0.000$ & $1.000{\pm}0.000$ & $0.988{\pm}0.009$ & $1.000{\pm}0.000$ \\
        \addlinespace
        Filtering & Square & $0.1$ & $0.760{\pm}0.006$ & $0.710{\pm}0.024$ & $0.730{\pm}0.013$ & $0.534{\pm}0.029$ & $0.364{\pm}0.026$ \\
                  &        & $\mathbf{0.2}$ & $0.762{\pm}0.005$ & $0.732{\pm}0.018$ & $0.734{\pm}0.014$ & $0.528{\pm}0.021$ & $0.360{\pm}0.024$ \\
                  &        & $0.5$ & $0.738{\pm}0.011$ & $0.700{\pm}0.019$ & $0.706{\pm}0.020$ & $0.536{\pm}0.031$ & $0.388{\pm}0.021$ \\
        \addlinespace
        Filtering & Transport & $0.1$ & $0.400{\pm}0.012$ & $0.370{\pm}0.019$ & $0.372{\pm}0.024$ & $0.246{\pm}0.015$ & $0.150{\pm}0.013$ \\
                  &           & $\mathbf{0.2}$ & $0.412{\pm}0.013$ & $0.420{\pm}0.020$ & $0.334{\pm}0.015$ & $0.268{\pm}0.009$ & $0.196{\pm}0.017$ \\
                  &           & $0.5$ & $0.398{\pm}0.010$ & $0.400{\pm}0.014$ & $0.402{\pm}0.031$ & $0.268{\pm}0.017$ & $0.148{\pm}0.015$ \\
        \addlinespace
        Selection & Square & $0.1$ & $0.592{\pm}0.019$ & $0.594{\pm}0.012$ & $0.628{\pm}0.016$ & $0.680{\pm}0.010$ & $0.714{\pm}0.013$ \\
                  &        & $\mathbf{0.2}$ & $0.572{\pm}0.034$ & $0.600{\pm}0.009$ & $0.634{\pm}0.015$ & $0.764{\pm}0.016$ & $0.758{\pm}0.010$ \\
                  &        & $0.5$ & $0.572{\pm}0.026$ & $0.592{\pm}0.011$ & $0.670{\pm}0.013$ & $0.666{\pm}0.014$ & $0.738{\pm}0.012$ \\
        \addlinespace
        Selection & Transport & $0.1$ & $0.382{\pm}0.013$ & $0.418{\pm}0.018$ & $0.488{\pm}0.027$ & $0.462{\pm}0.012$ & $0.400{\pm}0.010$ \\
                  &           & $\mathbf{0.2}$ & $0.344{\pm}0.011$ & $0.452{\pm}0.016$ & $0.484{\pm}0.021$ & $0.476{\pm}0.015$ & $0.458{\pm}0.009$ \\
                  &           & $0.5$ & $0.424{\pm}0.009$ & $0.396{\pm}0.013$ & $0.412{\pm}0.022$ & $0.475{\pm}0.016$ & $0.428{\pm}0.012$ \\
        \bottomrule
    \end{tabular}
\end{table}

Table~\ref{tab:lambda-sensitivity-sim2sim} extends this analysis to Sim-to-Sim co-training. Performance becomes more sensitive to $\lambda$ under simulator shift, especially at low and intermediate budgets. On \texttt{StackCube}, the success rate at the $10\%$ budget ranges from $32.9\%$ to $42.3\%$. On \texttt{CupCake}, $\lambda=0.5$ performs best at the $25\%$ budget. Across all 18 configurations, $\lambda=0.2$ obtains the highest or tied-highest result in 15 cases and the highest budget-averaged success rate on every task. These results reveal task-dependent sensitivity while supporting $\lambda=0.2$ as a stable shared configuration for the single-simulator and Sim-to-Sim settings.

\begin{table}[htbp]
    \centering
    \caption{Sensitivity to $\lambda$ in Sim-to-Sim co-training with $\gamma=0.01$. Entries are target-domain success rates reported as mean $\pm$ standard deviation. Bold parameter values denote the default configuration.}
    \label{tab:lambda-sensitivity-sim2sim}
    \footnotesize
    \setlength{\tabcolsep}{2.5pt}
    \renewcommand{\arraystretch}{1.10}
    \begin{tabular}{lccccccc}
        \toprule
        Task & $\lambda$ & $5\%$ & $10\%$ & $15\%$ & $25\%$ & $50\%$ & $75\%$ \\
        \midrule
        \texttt{Peg} & $0.1$ & $0.572{\pm}0.044$ & $0.588{\pm}0.045$ & $0.634{\pm}0.052$ & $0.678{\pm}0.010$ & $0.709{\pm}0.040$ & $0.674{\pm}0.013$ \\
            & $\mathbf{0.2}$ & $0.567{\pm}0.033$ & $0.614{\pm}0.026$ & $0.633{\pm}0.031$ & $0.716{\pm}0.023$ & $0.709{\pm}0.040$ & $0.674{\pm}0.013$ \\
            & $0.5$ & $0.564{\pm}0.069$ & $0.591{\pm}0.030$ & $0.632{\pm}0.045$ & $0.704{\pm}0.051$ & $0.709{\pm}0.040$ & $0.674{\pm}0.013$ \\
        \addlinespace
        \texttt{StackCube} & $0.1$ & $0.373{\pm}0.036$ & $0.329{\pm}0.054$ & $0.387{\pm}0.031$ & $0.369{\pm}0.031$ & $0.429{\pm}0.038$ & $0.426{\pm}0.029$ \\
                  & $\mathbf{0.2}$ & $0.399{\pm}0.040$ & $0.423{\pm}0.044$ & $0.432{\pm}0.028$ & $0.421{\pm}0.021$ & $0.455{\pm}0.015$ & $0.426{\pm}0.029$ \\
                  & $0.5$ & $0.387{\pm}0.035$ & $0.387{\pm}0.050$ & $0.375{\pm}0.029$ & $0.376{\pm}0.010$ & $0.429{\pm}0.038$ & $0.426{\pm}0.029$ \\
        \addlinespace
        \texttt{CupCake} & $0.1$ & $0.503{\pm}0.025$ & $0.513{\pm}0.029$ & $0.527{\pm}0.013$ & $0.536{\pm}0.033$ & $0.522{\pm}0.040$ & $0.495{\pm}0.035$ \\
                & $\mathbf{0.2}$ & $0.525{\pm}0.033$ & $0.522{\pm}0.045$ & $0.541{\pm}0.026$ & $0.545{\pm}0.033$ & $0.522{\pm}0.040$ & $0.495{\pm}0.035$ \\
                & $0.5$ & $0.505{\pm}0.025$ & $0.505{\pm}0.033$ & $0.523{\pm}0.039$ & $0.562{\pm}0.031$ & $0.522{\pm}0.040$ & $0.495{\pm}0.035$ \\
        \bottomrule
    \end{tabular}
\end{table}

\paragraph{Effect of $\gamma$.}
We next study the cone threshold in Sim-to-Sim co-training with $\lambda=0.2$. As shown in Table~\ref{tab:gamma-sensitivity-sim2sim}, $\gamma=0.01$ outperforms $\gamma=0.1$ in 12 of the 18 task and budget configurations and matches it in five more. The larger threshold produces a higher result only on \texttt{CupCake} at the $10\%$ budget, where the difference is $0.5\%$. Its lower performance on \texttt{Peg} and \texttt{StackCube} suggests that a narrow feasible cone can exclude trajectories with useful complementary effects. The setting $\gamma=0.01$ retains a broader pool of performance-aligned candidates and yields more consistent performance across tasks.

\begin{table}[htbp]
    \centering
    \caption{Sensitivity to $\gamma$ in Sim-to-Sim co-training with $\lambda=0.2$. Entries are target-domain success rates reported as mean $\pm$ standard deviation. Bold parameter values denote the default configuration.}
    \label{tab:gamma-sensitivity-sim2sim}
    \footnotesize
    \setlength{\tabcolsep}{2.5pt}
    \renewcommand{\arraystretch}{1.10}
    \begin{tabular}{lccccccc}
        \toprule
        Task & $\gamma$ & $5\%$ & $10\%$ & $15\%$ & $25\%$ & $50\%$ & $75\%$ \\
        \midrule
        \texttt{Peg} & $\mathbf{0.01}$ & $0.567{\pm}0.033$ & $0.614{\pm}0.026$ & $0.633{\pm}0.031$ & $0.716{\pm}0.023$ & $0.709{\pm}0.040$ & $0.674{\pm}0.013$ \\
            & $0.1$ & $0.559{\pm}0.007$ & $0.596{\pm}0.044$ & $0.610{\pm}0.054$ & $0.659{\pm}0.036$ & $0.709{\pm}0.040$ & $0.674{\pm}0.013$ \\
        \addlinespace
        \texttt{StackCube} & $\mathbf{0.01}$ & $0.399{\pm}0.040$ & $0.423{\pm}0.044$ & $0.432{\pm}0.028$ & $0.421{\pm}0.021$ & $0.455{\pm}0.015$ & $0.426{\pm}0.029$ \\
                  & $0.1$ & $0.364{\pm}0.026$ & $0.360{\pm}0.021$ & $0.421{\pm}0.039$ & $0.404{\pm}0.024$ & $0.440{\pm}0.015$ & $0.426{\pm}0.029$ \\
        \addlinespace
        \texttt{CupCake} & $\mathbf{0.01}$ & $0.525{\pm}0.033$ & $0.522{\pm}0.045$ & $0.541{\pm}0.026$ & $0.545{\pm}0.033$ & $0.522{\pm}0.040$ & $0.495{\pm}0.035$ \\
                & $0.1$ & $0.475{\pm}0.035$ & $0.527{\pm}0.031$ & $0.523{\pm}0.034$ & $0.540{\pm}0.052$ & $0.522{\pm}0.040$ & $0.495{\pm}0.035$ \\
        \bottomrule
    \end{tabular}
\end{table}

Sensitivity decreases at larger selection budgets, where several parameter values produce identical results. As the budget grows, the selected subset covers the feasible pool and the remaining positions follow the same utility ranking. Different configurations consequently converge to the same subset. The effects of $\lambda$ and $\gamma$ are more visible at lower budgets, where each selection must resolve a stronger trade-off among candidate trajectories.

\subsection{Robustness to Unseen Lighting Conditions.}
 We examine whether TUCO retains its advantage under an unseen visual shift by introducing an RGB light strip around the peg and hole during \texttt{Peg} evaluation
  (Figure~\ref{fig:sim2real-lightstrip}). We evaluate each policy over 20 trials. Table~\ref{tab:sim2real-peg-lighting} shows that TUCO achieves the highest success rate at 25.0\%, exceeding both CUPID and target-only training by 10 percentage points, while full-pool co-training fails every trial. Despite the substantial performance drop under altered lighting, TUCO maintains its advantage over the evaluated baselines.

\begin{figure}[htbp]
    \centering
    \begin{minipage}[c]{0.34\textwidth}
        \centering
        \includegraphics[height=1.4in]{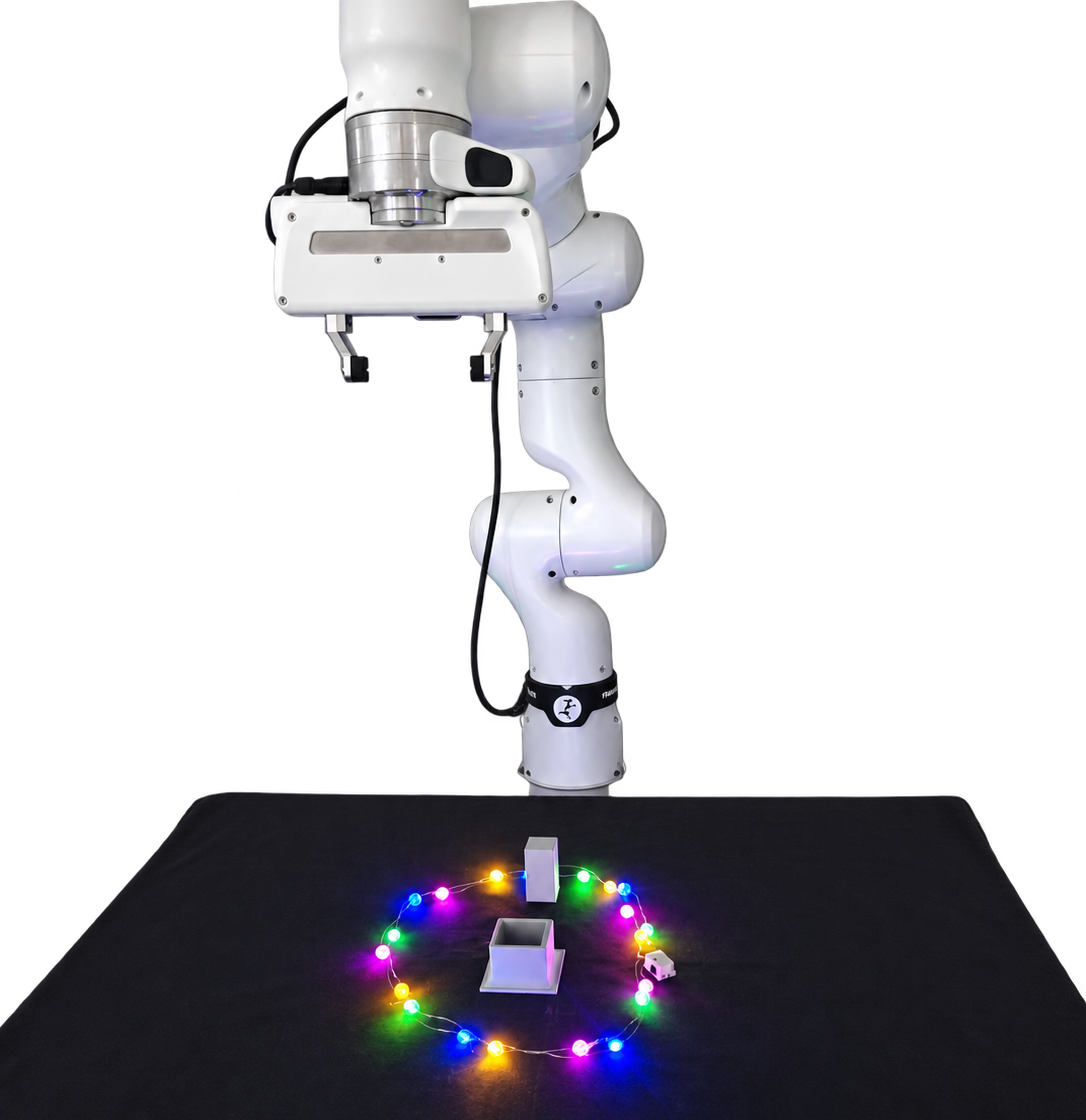}
        \captionof{figure}{Physical \texttt{Peg} setup with the RGB light strip enabled at test time. The strip surrounds the peg and hole and changes illumination and color within the workspace without moving any object.}
        \label{fig:sim2real-lightstrip}
    \end{minipage}\hfill
    \begin{minipage}[c]{0.60\textwidth}
        \centering
        \captionof{table}{\texttt{Peg} success rates (\%) with the RGB light strip enabled at test time. The light-strip column uses 20 trials per method. The standard-lighting column repeats the 40-trial entries of Table~\ref{tab:sim2real-tasks}.}
        \label{tab:sim2real-peg-lighting}
        \small
        \setlength{\tabcolsep}{8pt}
        \begin{tabular}{@{}lcc@{}}
            \toprule
            Method & Standard lighting & RGB light strip \\
            \midrule
            10 real demos & 15.0 & 15.0 \\
            10 real demos + all sim data & 17.5 & 0.0 \\
            CUPID & 75.0 & 15.0 \\
            TUCO (ours) & \textbf{82.5} & \textbf{25.0} \\
            \bottomrule
        \end{tabular}
    \end{minipage}
\end{figure}

\section{Extended Discussions}
\label{app:extended-discussion}

\subsection{Insights into Robust Curation across Domain Gaps}

Our study shows that the same TUCO criterion remains effective in single-simulator, Sim-to-Sim, and Sim-to-Real settings. This result supports grounding curation in observed target-domain behavior as the domain gap changes. Source-domain quality, latent similarity, and trajectory geometry can provide useful proxies, but they may increasingly reflect rendering, sensing, or dynamics differences across domains. TUCO represents each source demonstration through its estimated effects on shared target-domain scoring rollouts. In Sim-to-Real, utility measures aggregate support for recorded successful executions, while coverage measures complementary influence patterns across them. This feedback supports the same decomposition and selection mechanism using a simulation-pretrained base policy.

Our study also reveals that performance grounding alone is insufficient when the output is a dataset. The value of a source trajectory depends on the fixed target demonstrations and the influence modes already represented by the selected subset. Independent ranking can repeatedly select trajectories with similar training effects. TUCO addresses this dependency through set-level marginal coverage while retaining direct target-performance guidance. The consistency observed across domain gaps therefore supports two general design principles for curation: evaluate data through downstream target behavior and optimize the composition of the selected set.

\subsection{Source Composition as a Co-Training Design Variable}

To the best of our knowledge, this study provides the first systematic evidence that active curation improves both Sim-to-Sim and Sim-to-Real co-training. More broadly, it introduces a new perspective on co-training: in addition to choosing the source-to-target mixing ratio, one should decide which source demonstrations occupy the source portion of that mixture.   The source-to-target demonstration ratio specifies the relative amounts
of available training data, while curation determines which source behaviors are included. Our comparisons with the full source pool establish the benefit of selecting a subset. Comparisons with random selection at the same source budget and with fixed target demonstrations further isolate the value of source-data composition.

This distinction is particularly important for large simulation pools. Simulation can cover behaviors missing from scarce target data, but it can also overrepresent redundant trajectories and domain-specific strategies. Active curation treats this pool as a candidate library and retains demonstrations that support target success while complementing existing target data. Our findings therefore shift the co-training question from only \emph{how much simulation data should be mixed} toward \emph{which simulation demonstrations should be mixed}. This additional design dimension turns simulation scale into more effective cross-domain supervision and creates a path toward stronger policies with lower training and real-data collection costs.

\subsection{Limitations}
\paragraph{Automatic curation budget.}
TUCO optimizes which source demonstrations to retain under a prescribed curation budget and source-to-target mixing ratio. This design isolates the effect of source-data composition and supports controlled comparison across curation methods. In deployment, however, the appropriate subset size is typically unknown and may vary with the domain gap and the behavioral support already provided by the target demonstrations. A small budget may omit useful source behaviors, while a large budget may reintroduce redundancy or source-specific patterns after the most valuable candidates have been selected. TUCO evaluates set-dependent marginal value under a fixed cardinality constraint and therefore provides no calibrated boundary between demonstrations that should be retained and those that provide no further benefit. This limitation reflects a common formulation in data curation, where candidates receive selection scores and the final subset is determined by a predefined budget. We believe that it is an important direction to
automatically identify the useful portion of a source pool, independent of a predefined retention ratio. We leave joint optimization of subset composition and cardinality for future work.

\paragraph{Broader transfer settings.}
Our evaluation covers manipulation tasks across single-simulator, Sim-to-Sim, and Sim-to-Real settings with compatible observation and action interfaces. This scope allows source and target demonstrations to train the same policy and makes their influence profiles directly comparable. Cross-task or cross-embodiment transfer may additionally require aligning observations, actions, and task semantics before source demonstrations can be evaluated in a shared closed-loop space. Large vision-language-action policies may also produce different influence geometry from the task-specific policies studied here. We also leave adaptation and evaluation of TUCO under these broader transfer settings for future work.